\documentclass[sigconf, nonacm]{acmart}
\usepackage{caption}
\usepackage[ruled,vlined]{algorithm2e}
\usepackage{booktabs}
\usepackage{pifont}
\usepackage{fancyhdr} % 导入fancyhdr包
\usepackage{color}
\usepackage{soul, framed}
\usepackage{makecell}
\usepackage{xcolor}
\usepackage{colortbl}
\usepackage{graphicx}
\usepackage{textcomp}
\usepackage{xcolor,colortbl}
\usepackage{multirow}
\usepackage{bm,dsfont}
\usepackage{subfigure}
\usepackage{balance}
\usepackage{stmaryrd}
\usepackage{autobreak}
\usepackage{enumitem}
\usepackage{bbding}
\usepackage{pifont}
\usepackage{wasysym}
\usepackage{tikz}
\usepackage{pgfplots}
\pgfplotsset{compat=1.17}
\usetikzlibrary{positioning,arrows,shapes,calc}
\usepackage{tcolorbox}

\usepackage{thm-restate} 

\newtheorem{definition}{Definition}

\newtheorem{example}{Example}

\newcommand{\frameworkname}{HyperJoin}
\newcommand{\modelname}{HIN}
\newcommand{\modelfullname}{Hierarchical Interaction Network}

\newcommand{\HIname}{Hyperedge Interaction}
\newcommand{\HRname}{Hyperedge Refinement}
\newcommand{\HIabbr}{HI}
\newcommand{\HRabbr}{HR}

\newcommand\vldbdoi{10.14778/3665844.3665853}
\newcommand\vldbpages{2227 - 2240}
\newcommand\vldbvolume{17}
\newcommand\vldbissue{9}
\newcommand\vldbyear{2024}
\newcommand\vldbauthors{\authors}
\newcommand\vldbtitle{\shorttitle} 
\newcommand\vldbavailabilityurl{URL_TO_YOUR_ARTIFACTS}
\newcommand\vldbpagestyle{plain} 
\begin{document}
\title{HyperJoin: LLM-augmented Hypergraph Link Prediction for Joinable Table Discovery}

\author{Shiyuan Liu}
\affiliation{%
  \institution{University of Technology Sydney}
  }
\email{shiyuan.liu-1@student.uts.edu.au}
\orcid{0000-0003-2424-8337}

\author{Jianwei Wang}
\affiliation{%
  \institution{University of New South Wales}
  }
\email{jianwei.wang1@unsw.edu.au}
\orcid{0009-0000-7887-4179}

\author{Xuemin Lin}
\affiliation{%
 \institution{ACEM, Shanghai Jiao Tong University}
  }
\email{xuemin.lin@sjtu.edu.cn}
\orcid{0000-0003-2396-7225}

\author{Lu Qin}
\affiliation{%
 \institution{University of Technology Sydney}
  }
\email{lu.qin@uts.edu.au}
\orcid{0000-0001-6068-5062}

\author{Wenjie Zhang}
\affiliation{%
  \institution{University of New South Wales}
  }
\email{wenjie.zhang@unsw.edu.au}
\orcid{0000-0001-6572-2600}

\author{Ying Zhang}
\affiliation{%
  \institution{University of Technology Sydney}
  }
\email{ying.zhang@uts.edu.au}
\orcid{0000-0002-2674-1638}

\begin{abstract}
% Joinable table discovery has aroused widespread interest in data lake management.
As a pivotal task in data lake management, joinable table discovery has attracted widespread interest.
While existing language model-based methods achieve remarkable performance by combining offline column representation learning with online ranking, their design insufficiently accounts for the underlying structural interactions: (1) offline, they directly model tables into isolated or pairwise columns, thereby struggling to capture the rich inter-table and intra-table structural information;
% fail to exploit structure prior by neglecting both inter-table and intra-table dependencies, 
% and (2) suffer from result fragmentation due to overlooking retrieval coherence.
and (2) online, they rank candidate columns based solely on query-candidate similarity, ignoring the mutual interactions among the candidates, leading to incoherent result sets.
To address these limitations, we propose \frameworkname{}, a large language model (LLM)-augmented \textbf{\underline{Hyper}}graph framework for \textbf{\underline{Join}}able table discovery.
% that jointly models inter-table joinability and intra-table attribute dependencies.
Specifically, we first construct a hypergraph to model tables using both the intra-table hyperedges and the LLM-augmented inter-table hyperedges. Consequently, the task of joinable table discovery is formulated as link prediction on this constructed hypergraph.
We then design \modelname{}, a \modelfullname~that learns expressive column representations by integrating global message passing across hyperedges with local message passing between columns and hyperedges.
To strengthen coherence and internal consistency in the result columns, we cast online ranking as a coherence-aware top-$K$ column selection problem. 
We then introduce a reranking module that leverages a maximum spanning tree algorithm to prune noisy connections and maximize coherence.
Experiments demonstrate the superiority of \frameworkname{}, achieving average improvements of 21.4\% (Precision@15) and 17.2\% (Recall@15) over the best baseline.

\end{abstract}


\maketitle

\pagestyle{\vldbpagestyle}
\begingroup\small\noindent\raggedright\textbf{PVLDB Reference Format:}\\
\vldbauthors. \vldbtitle. PVLDB, \vldbvolume(\vldbissue): \vldbpages, \vldbyear.\\
\href{https://doi.org/\vldbdoi}{doi:\vldbdoi}
\endgroup
\begingroup
\renewcommand\thefootnote{}\footnote{\noindent
This work is licensed under the Creative Commons BY-NC-ND 4.0 International License. Visit \url{https://creativecommons.org/licenses/by-nc-nd/4.0/} to view a copy of this license. For any use beyond those covered by this license, obtain permission by emailing \href{mailto:info@vldb.org}{info@vldb.org}. Copyright is held by the owner/author(s). Publication rights licensed to the VLDB Endowment. \\
\raggedright Proceedings of the VLDB Endowment, Vol. \vldbvolume, No. \vldbissue\ %
ISSN 2150-8097. \\
\href{https://doi.org/\vldbdoi}{doi:\vldbdoi} \\
}\addtocounter{footnote}{-1}\endgroup

\ifdefempty{\vldbavailabilityurl}{}{
\vspace{.3cm}
\begingroup\small\noindent\raggedright\textbf{PVLDB Artifact Availability:}\\
The source code, data, and/or other artifacts have been made available at \url{https://github.com/T-Lab/HyperJoin}.
\endgroup
}

% \vspace{-3mm}
\section{Introduction}
\label{sec:Introduction}

\begin{figure}
\vspace{0.5pt}
\centering
\includegraphics[width=1\linewidth]{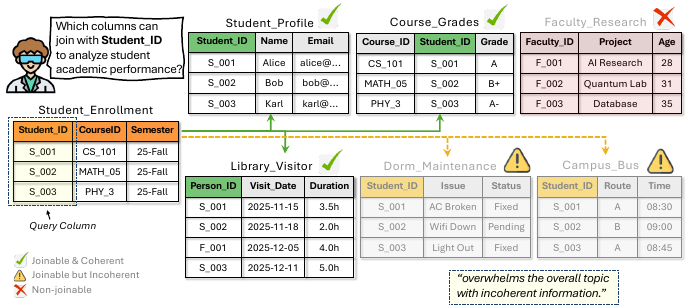}
\caption{An example of joinable table discovery in data lake.}
% Given a query column \texttt{Customer\_ID} from \texttt{Sales\_Report}, the system searches through the data lake to identify joinable columns (marked with green checkmarks) and returns a coherent top-K result set that forms a consistent join graph.}
\label{fig:example}
\end{figure}

The proliferation of open data portals and enterprise data lakes has led to massive collections of tabular data, creating tremendous opportunities for data analysis and machine learning. 
However, these tables are often fragmented, making it difficult to leverage complementary information from other tables to enhance analysis.
Prior studies~\cite{chai2020human,lai2025auto,deng2017data,deng2024misdetect,nargesian2019data} highlight the prevalence of this concern, reporting that approximately 89\% of real-world business intelligence projects involve multiple tables that require joins.
In this work, we investigate joinable table discovery. Given a table repository $\mathcal{T}$ and a query column $C_q$, joinable table discovery aims to find columns $C$ from $\mathcal{T}$ that exhibit a large number of semantically matched cells with $C_q$. Thus, the table containing column $C$ can be joined with the query table to enrich features and enhance downstream data analysis or machine learning tasks.

\begin{example}
As illustrated in Figure~\ref{fig:example}, a data scientist analyzes student academic performance using the \texttt{Student\allowbreak \_Enrollment} table and searches for columns joinable with the \texttt{Student\_ID} ($C_q$).
The data lake contains several semantically joinable columns (green checkmarks).
However, columns from tables like \texttt{Dorm\_Maintenance} and \texttt{Campus\_Bus}, while joinable, originate from unrelated service domains and may overwhelm the analysis with incoherent information (yellow warnings).
An effective join discovery method should therefore identify joinable columns while ensuring that the returned results remain coherent to better support the analysis.
\end{example}

\begin{figure*}[t]
\centering
\includegraphics[width=0.99\linewidth]{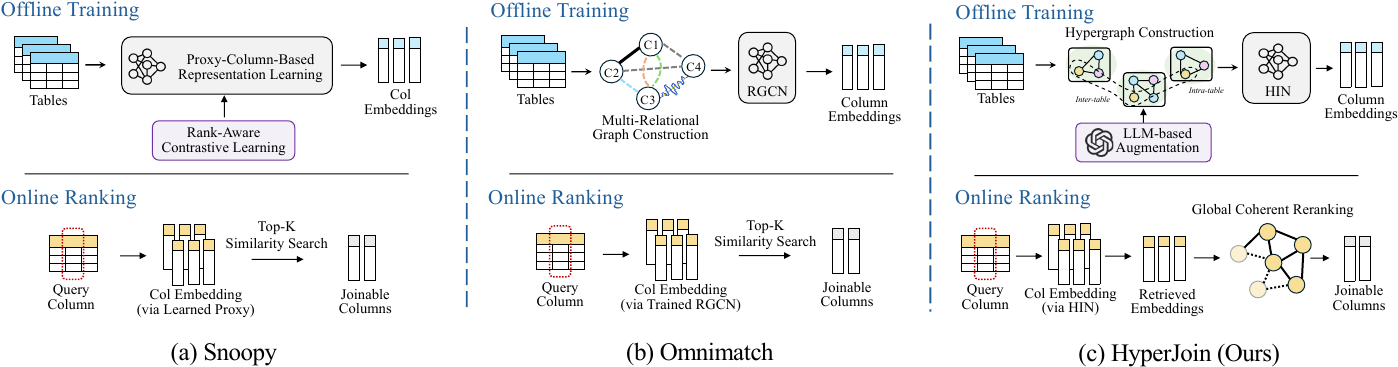}
\caption{Framework comparisons of joinable table discovery methods.}
\vspace{-3mm}
\label{fig:comparison}
\end{figure*}

Recognizing the importance and practical value of joinable table search, a set of methods has been developed (refer to the survey and benchmark paper~\cite{deng2024lakebench} for an in-depth discussion).
% In the early stage, xxx
% {\color{red}TODO: add the introduction of non-learning methods}.
{\color{black}
Early non-learning methods generally cast joinable table discovery as a set-similarity search problem \cite{khatiwada2022integrating}. 
Representative approaches, such as JOSIE~\cite{zhu2019josie} and LSH Ensemble~\cite{zhu2016lsh}, model each column as a set of distinct values and identify joinable columns based on their overlap, supported by inverted indexes or MinHash-based LSH for efficiency.}
However, these methods limit the join type to equi-join based on exact string matching, overlooking numerous semantically matched cells (e.g., ``NY'' semantically matches ``New York''), thus underestimating column joinabilities. 

Recently, pre-trained language models (PLMs)-based methods are introduced to exploit the semantic capabilities of PLMs (e.g., BERT~\cite{devlin2019bert} and DistilBERT~\cite{sanh2019distilbert}) for more effective joinable table discovery.
Representative methods include DeepJoin~\cite{dong2023deepjoin}, Snoopy~\cite{guo2025snoopy} and Omnimatch~\cite{koutras2025omnimatch}.
As illustrated in Figure~\ref{fig:comparison}(a) and Figure~\ref{fig:comparison}(b), these methods typically integrate an offline training phase with an online ranking phase to optimize both efficiency and effectiveness. During the offline phase, a model is trained to learn column representations, encoding the content of each column into a dense vector. In the online phase, given a query column, the system identifies joinable columns based on the similarity between the query and candidate columns.

% \vspace{1mm}
\noindent \textbf{Motivations.}
Although these PLMs-based methods achieve promising performance, their designs insufficiently account for the underlying structural interactions within tables, which are critical for accurate joinable table discovery. Specifically: 

Firstly, during the offline training phase, these methods typically partition tables into isolated or pairwise columns.
This approach struggles to capture rich intra-table structural information (such as high-order dependencies among columns and the hierarchical relationships between tables and columns) as well as inter-table structural information (such as global join topologies).
As a result, the learned representation lacks sufficient expressiveness to capture these intricate dependencies, leading to suboptimal accuracy. 
Specifically, DeepJoin~\cite{dong2023deepjoin} and Snoopy~\cite{guo2025snoopy} process columns individually to learn their representations. Although Omnimatch~\cite{koutras2025omnimatch} attempts to utilize a graph to learn to model the column relationships. However, it still primarily focuses on pairwise relationships. 

Secondly, in the online ranking phase, existing methods typically perform simple top-$K$ ranking based solely on the query-candidate similarity scores. This approach neglects the mutual correlations and latent dependencies among the candidate columns. Consequently, this yields incoherent result sets. Integrating such columns may introduce notable noise, which overwhelms the overall insights with irrelevant information. For example,
as shown in Figure~\ref{fig:example}, when the query is \texttt{Student\_ID}, a Top-K retrieval may include columns from \texttt{Campus\_Bus} and \texttt{Dorm\_Maintenance}.
Although these columns are joinable, they should be ranked lower since they originate from unrelated domains and fail to support the intended analysis of student academic performance.

\noindent\textbf{Challenges.} To design an accurate joinable table discovery method, two challenges exist below:

\textit{Challenge I: How to capture the complex structural relationships with a unified framework for effective offline representation learning?}
Table repositories exhibit complex structural relationships comprising both inter-table and intra-table relationships. 
Moreover, these relationships are not isolated; rather, they are strongly interdependent. For example, inter-table column joinability may only be meaningful when contextualized with fine-grained intra-table semantics, making isolated modeling approaches inadequate. Therefore, a significant challenge lies in devising a unified framework that jointly employs these complex structural relationships for joinable table discovery.

\textit{Challenge II: How to effectively and efficiently search globally relevant and coherent joinable tables from large-scale data lakes?} 
Data lakes typically contain hundreds to thousands of heterogeneous tables. When accounting for mutual context, the number of potential join paths grows exponentially. Furthermore, as we proved in Theorem~\ref{thm:np_hard_formulation}, identifying the optimal $K$-column subset that maximizes joint coherence is NP-hard, causing exhaustive search to be intractable for large-scale data lakes.

% \vspace{1mm}
\noindent \textbf{\color{black}{Our approaches.}}
Driven by the aforementioned challenges, we propose \frameworkname{}, a new \textbf{\underline{Hyper}}graph-based framework that recasts the problem of \textbf{\underline{Join}}able table search as link prediction on the hypergraph to better leverage the structural context of tables.
% for accurate \textbf{\underline{Join}}able table discovery. 
As illustrated in Figure~\ref{fig:comparison}(c), similar to previous PLM-based methods, \frameworkname{} also employs a two-phase architecture: an offline representation learning phase and an online ranking phase.

{\color{black}To address Challenge I, in the offline phase, we propose to use a hypergraph to model the structure prior of tables and learn column representations. Formally, the hypergraph is formulated with columns serving as nodes and comprising two types of hyperedges: (1) \textit{Large language model (LLM)-augmented inter-table hyperedges} and (2) \textit{Intra-table hyperedges}. Specifically, the LLM-augmented hyperedges connect all columns that belong to the same connected component of joinable columns, explicitly modeling inter-table information of joinability.
We first employ LLMs to augment data by generating semantically equivalent column name variants, and then treat each original column together with its LLM-generated variants as a unified join-key entity, using this expanded set to construct inter-table hyperedges that propagate joinability signals across tables and their LLM-generated alternatives. 
In addition, the intra-table hyperedges connect all columns within the same table, capturing high-order relationships among columns and column-table relationships. 
Furthermore, we design \modelname, a \modelfullname~that performs hierarchical message passing aligned with our two types of hyperedges: it first aggregates column information within intra-table and LLM-augmented inter-table hyperedges, and then enables all-to-all interaction across hyperedges through a global mixing step, allowing each column to better leverage both the intra-table schema relationships and inter-table joinability signals.

{\color{black}To solve Challenge II, in the online phase, we introduce a global coherent reranking module to improve both efficiency and accuracy. Specifically, we first formulate the task of global coherent joinable table selection as a coherence-aware top-$K$ column selection problem that balances individual query relevance with inter-candidate coherence.
We prove that this problem is NP-hard. 
Thus, the proposed global coherent reranking module first constructs a weighted complete graph spanning the query and target columns to capture the structural context of columns after the first-stage top-B retrieval ($K<B$). An efficient greedy approximation algorithm that leverages a Maximum Spanning Tree (MST) algorithm are employed to prune noisy connections and maximize the global matching score.

% \vspace{1mm}
\noindent \textbf{Contributions.}
The main contributions are as follows:
\begin{itemize}[leftmargin=10pt, topsep=1pt]
\item{} We propose \frameworkname, a new hypergraph-based framework for joinable table discovery that exploits structure prior via dual-type hypergraphs. It contains the offline representation learning phase and the online ranking phase.

\item{} In the offline phase, we build a hypergraph with LLM-augmented inter-table hyperedges and intra-table hyperedges, and design the \modelname~for effective message passing and learn expressive column representation.

\item{} In the online phase, we propose a global coherent reranking module that leverages an efficient greedy approximation algorithm with the maximum spanning tree to tractably discover the coherent result columns.

\item{} Experiments across 4 benchmarks show the superiority of \frameworkname, achieving average improvements of 21.4\% in Precision@15 and 17.2\% in Recall@15 over previous state-of-the-art methods.

\end{itemize}

\section{Preliminaries}
\label{sec:Preliminary}

\begin{table}[t]
\centering %\small %\scriptsize
\caption{Symbols and Descriptions}
\label{tab:symbol}
\begin{tabular}{|p{1.8cm}|p{6.0cm}|}
% {|p{3.2cm}|p{4.7cm}|} {|c|c|}
\hline
\cellcolor{lightgray}\textbf{Notation} & \cellcolor{lightgray}\textbf{Description} \\ \hline
$\mathcal{T}$ & table repository\\ \hline
${T}$ & a table, consisting of multiple columns\\ \hline
$\mathcal{C}$ & column repository (all columns in $\mathcal{T}$) \\ \hline
$C$ & a column, represented as a set of cell values\\ \hline
$C_q$ & query column \\ \hline
$|C|$ & cardinality (number of cells) of column $C$ \\ \hline
$\mathbf{h}_C \in \mathbb{R}^d$ & embedding of column $C$ \\ \hline
$J(C_q, C_t)$ & joinability score between $C_q$ and $C_t$\\ \hline 
$K$ & number of results to return\\ \hline                         $\mathcal{R}, \mathcal{R}^*$ & result set / optimal result set of size $K$\\ \hline
$\mathcal{H}(V, E)$ & hypergraph with nodes $V$ and hyperedges $E$\\ \hline
$E_{\text{inter}}$ & inter-table hyperedges\\ \hline
$E_{\text{intra}}$ & intra-table hyperedges\\ \hline
$f^{\theta}(\cdot)$ & neural network model with parameters $\theta$\\ \hline
$\lambda$ & coherence weight in reranking objective\\ \hline

\end{tabular}
\end{table}

\subsection{Problem Definition}

We follow the typical setting of joinable table discovery
\cite{dong2021efficient, dong2023deepjoin, fan2023semantics, deng2024lakebench, guo2025snoopy}.
% {\color{red} TODO: Add citations, add the notation and concepts of tables}
A data lake contains a table repository $\mathcal{T} = \{T_1, T_2, \dots, T_M\}$ with $M$ tables.
Each table $T_i$ consists of $n_i$ rows (tuples) and $m_i$ columns (attributes), where each cell value is denoted as $c_{ij}$.
We focus on textual columns and extract all of them from $\mathcal{T}$ into a column repository $\mathcal{C}=\{C_1, C_2, \dots, C_N\}$ containing $N$ columns across all tables.
Each column $C$ is represented as a set of cell values $\{c_1, c_2, \dots, c_{|C|}\}$ and may also include associated metadata such as column names, where $|C|$ denotes the cardinality (number of cells) of column $C$.
We use $C_q$ to denote the query column.
$\mathcal{R}^*$ and $\mathcal{R}$ are utilized to denote the ground-truth optimal result set and a predicted result set, respectively.
Next, we give the formal definition of joinable table discovery.

\begin{definition}[Pairwise Column Joinability]
\label{def:pairwise_joinability}
% Building upon value correspondence, we define the joinability between two columns as a measure of their content overlap. This captures the \textit{relevance} of a candidate column to a query.
% {\color{red} TODO: Be concise, it seems that the above is unnecessary.}
Given a query column $C_q$ and a candidate column $C_t$ from the repository $\mathcal{C}$, their pairwise joinability score $J(C_q, C_t)$ is defined as the average best-match similarity for each value in the query column:
% {\color{red} TODO: Be concise.}
\begin{equation}
\label{eq:joinability}
J(C_q, C_t) = \frac{1}{|C_q|} \sum_{v_q \in C_q} \max_{v_t \in C_t} \text{sim}(v_q, v_t)
\end{equation}
% This score is directional and reflects how well the content of $C_q$ is covered by the content of $C_t$.
This directional score measures how well $C_q$ is covered by $C_t$. $\text{sim}(v_q, v_t) \in [0,1]$ is the semantic value correspondence that quantifies the semantic relatedness between two cell values $v_i$ and $v_j$.
\end{definition}

\begin{definition}[Joinable Table Discovery]
\label{def:jtd}
Given a query column $C_q$, a table repository $\mathcal{T}$, and an integer $K$, the task of joinable table discovery is to select a result set $\mathcal{R}^* \subseteq \mathcal{C}$ of size $K$ that each $C_t \in \mathcal{R}$ has a high joinability with $C_q$.

\end{definition}

\begin{figure*}[t]
  \centering
  \includegraphics[width=0.96\linewidth]{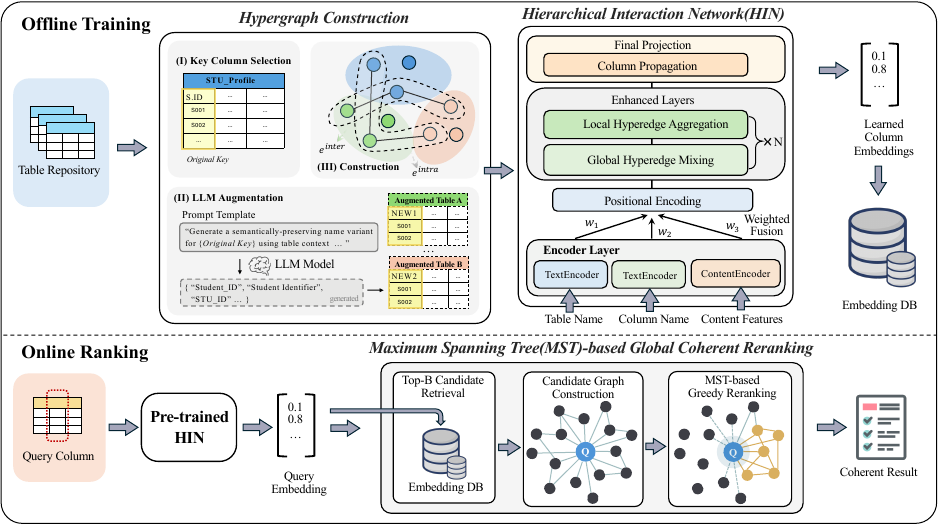}
  \vspace{-2mm}
  \caption{Framework overview of HyperJoin.}
  \vspace{-4mm}
  \label{fig:framework}
\end{figure*}

\subsection{Hypergraph}
% \subsection{Hypergraph and HGNNs}

% \textbf{Hypergraph.}
A hypergraph is a generalization of a graph in which an edge (called a \emph{hyperedge}) can connect \emph{any number} of vertices rather than only two \cite{lee2020hypergraph,arafat2023neighborhood}.
Formally, a hypergraph $\mathcal{H}=(V,E)$ consists of a vertex set $V$ and a hyperedge set $E$, where each hyperedge $e\in E$ denotes a \emph{group-wise} relation among multiple vertices (i.e., $\emptyset \neq e \subseteq V$). 
This capability makes hypergraphs a natural modeling tool for higher-order interactions in real-world data, such as co-authorship \cite{wu2022hypergraph} and biological complexes \cite{lugo2021classification}.

\subsection{Language Model and State of the Art}

Language models serve as the cornerstone of modern Natural Language Processing (NLP), fundamentally aiming to model the probability distribution of text sequences \cite{khattab2020colbert}.
A prevailing architecture in this domain is the auto-regressive (AR) model, which generates sequences by iteratively predicting the next token conditioned on the preceding context \cite{wang2024missing,wang2025llm}.
Building upon these foundations, the field has evolved from PLMs to LLMs \cite{zhang2023large,wang2024efficient,wang2024neural}.

PLMs, such as the auto-encoding BERT \cite{devlin2019bert} and its variants DistilBERT \cite{sanh2019distilbert} and MPNet \cite{song2020mpnet}, revolutionized NLP by learning contextualized representations from massive text corpora. These models map textual sequences into dense vectors that capture semantic similarity, typically requiring task-specific fine-tuning for downstream applications. In contrast, LLMs (e.g., the GPT series~\cite{radford2018improving}) represent a significant leap in scale, often containing tens or hundreds of billions of parameters. While LLMs share the pre-training paradigm with PLMs, their massive scale unlocks ``emergent abilities'', such as zero-shot reasoning and in-context learning, enabling them to solve complex tasks without explicit parameter updates.

\noindent\textbf{State of the Art.}
Current state-of-the-art solutions for joinable table discovery are mostly PLM-based, which follow a common offline/online paradigm.
In the offline stage, a model is trained (or fine-tuned), and embeddings are precomputed for all columns in the repository.
An approximate nearest neighbor index (e.g., IVFPQ \cite{jegou2010product} or HNSW \cite{malkov2018efficient} in the FAISS library) is often built for efficient retrieval \cite{wang2024simpler}.
In the online stage, the query column is encoded using the same model, and its top-$k$ most similar columns are retrieved from the index as the joinable tables. 

Despite this shared framework, these methods vary in how they obtain column embeddings.
DeepJoin~\cite{dong2023deepjoin} fine‑tunes the DistilBERT to embed columns into fixed‑dimensional vectors. 
It transforms column contents and metadata into a textual sequence using a set of contextualization options.
Snoopy~\cite{guo2025snoopy} introduces proxy columns to overcome limitations of direct PLM encoding. It learns a set of proxy column embeddings through rank‑aware contrastive learning. 
Omnimatch~\cite{koutras2025omnimatch} constructs a similarity graph where nodes represent columns and edges encode multiple signals.
It uses a Relational Graph Convolutional Network (RGCN) to learn column embeddings that aggregate neighborhood information. 
% \vspace{-4mm}
\section{Overview}
\label{sec:overview}

The overall architecture of \frameworkname~is illustrated in Figure~\ref{fig:framework}.
Instead of treating tables as isolated or pairwise columns, we model the entire data lake as a hypergraph, where nodes represent columns, and hyperedges explicitly encode both intra-table schema constraints and inter-table correlations. Consequently, we recast the joinable table discovery task as a link prediction problem on this hypergraph.
To effectively solve this link prediction task, we design the \modelname. This network learns expressive column embeddings by performing hierarchical message passing.
It consists of three key components: (1) a positional encoding module that captures table identity and global structural roles; (2) a Local Hyperedge Aggregation module that propagates information within the specific intra- and inter-table contexts; and (3) a Global Hyperedge Mixing module that enables all-to-all interaction across hyperedges.
In the online search phase, we introduce a global coherent reranking module to select the joinable column set. We formulate the task as a coherence-aware top-$K$ column selection problem on the candidate graph and solve it using an MST-based algorithm. 

% Section 4: Offline Representation Learning
\vspace{-2mm}
\section{Offline Representation Learning}
\label{sec:pretrain}

In this section, we present the offline representation learning phase of \frameworkname.
The section is organized into two main components.
First, Section~\ref{sec:hypergraph_construction} describes how we construct a unified hypergraph from the data lake, including initial column feature encoding, LLM-augmented inter-table hyperedge and intra-table hyperedge construction.
Second, Section~\ref{sec:model_design} presents our \modelfullname~that operates on the constructed hypergraph, detailing the network architecture and training objectives.

\vspace{-2mm}
\subsection{Hypergraph Construction}
\label{sec:hypergraph_construction}
To more comprehensively capture the structure context of the table lake, we model the data lake as a hypergraph where a single hyperedge can connect multiple columns simultaneously, enabling unified modeling of these complex structural relationships.

\noindent\textbf{Hypergraph Definition.}
Formally, we have a hypergraph
% \begin{equation}
% {\color{blue}
$\mathcal{H} = (\mathcal{V}, \mathcal{E}, \mathbf{X}^v, \mathbf{X}^e, \boldsymbol{\Pi})$,
% }
% \end{equation}
where $\mathcal{V} = \{v_1, \ldots, v_N\}$ denotes the set of $N$ nodes, each corresponding to a column in the data lake,
and $\mathcal{E} = \{e_1, \ldots, e_M\}$ denotes the set of $M$ hyperedges.
% {\color{blue}
Each node $v_i \in \mathcal{V}$ is associated with an initial feature vector $\mathbf{x}_{v_i}^v \in \mathbb{R}^d$,
and $\mathbf{X}^v \in \mathbb{R}^{N \times d}$ stacks all node features row-wise.
Each hyperedge $e_j \subseteq \mathcal{V}$ represents a column context (e.g., intra-table or inter-table) and is optionally associated with an initial hyperedge feature vector
$\mathbf{x}_{e_j}^e \in \mathbb{R}^{d_e}$, where $\mathbf{X}^e \in \mathbb{R}^{M \times d_e}$ stacks all hyperedge features row-wise.
% }
% {\color{red}What is Z?}
Hyperedges are allowed to overlap, meaning a column may participate in multiple contextual groups.
The hypergraph structure is represented by a node--hyperedge incidence matrix
$\boldsymbol{\Pi} \in \{0,1\}^{N \times M}$, where $\boldsymbol{\Pi}_{ij} = 1$ if $v_i \in e_j$ and $0$ otherwise.

% {\color{blue}
\noindent\textbf{Hypergraph Construction.}
The construction of $\mathcal{H}$ proceeds in three steps: (1) we encode initial node features $\mathbf{X}^v$ by fusing semantic signals from schema names and data content; (2) we construct intra-table hyperedges $\mathcal{E}_{\text{intra}}$ based on physical table schemas; and (3) we construct LLM-augmented inter-table hyperedges $\mathcal{E}_{\text{inter}}$ by identifying join-connected column groups, augmented with LLM-generated semantic variants.
% }
% \subsubsection{Initial Column Feature Encoding}

\label{sec:initial_encoding}
\noindent\textit{Step 1: Initial Column Feature Encoding.}
% \noindent\textbf{Motivation.}
Columns in a data lake contain rich semantic information from multiple sources: the table name provides high-level context, the column name describes specific semantics, and the cell values offer direct evidence.
A robust initial representation should fuse all these sources to provide a strong foundation for the subsequent structure-aware learning.

Each column node $v_i$ is initialized with a feature vector $\mathbf{x}_{v_i}^v$ by fusing three components: table name, column name, and cell values.
We encode the table name and column name using two independent text encoders, producing $\mathbf{h}^{\text{table}}_i \in \mathbb{R}^{d_0}$ and $\mathbf{h}^{\text{col}}_i \in \mathbb{R}^{d_0}$.
For cell values, we obtain a content representation by aggregating pre-trained value embeddings with simple statistical pooling and passing the result through a two-layer MLP, yielding $\mathbf{h}^{\text{content}}_i \in \mathbb{R}^{d_0}$.
Let $\mathcal{S}=\{\texttt{table},\texttt{col},\texttt{content}\}$ denote the components and $\mathbf{w}\in\mathbb{R}^{3}$ be learnable fusion weights; we compute
\begin{equation}
\mathbf{h}^{(0)}_i = \sum_{s\in\mathcal{S}} \text{softmax}(\mathbf{w})_s \cdot \mathbf{h}^{s}_i,
\end{equation}
followed by a linear projection to obtain 
% {\color{blue}
$\mathbf{x}_{v_i}^v=\mathbf{W}_0\mathbf{h}^{(0)}_i \in \mathbb{R}^{d}$.
Stacking  $\mathbf{x}_{v_i}^v$ yields the node feature matrix $\mathbf{X}^v\in\mathbb{R}^{N\times d}$ used in $\mathcal{H}$.
% \subsubsection{Intra-Table Hyperedges}

\label{sec:intra_hyperedges}
\noindent\textit{Step 2: Intra-Table Hyperedges.}
For each table $T$, we construct an intra-table hyperedge
$e^{\text{intra}}_T = \{v_i \mid v_i \text{ corresponds to a column in } T\}$.
This is motivated by two points:
(i) tables naturally group semantically related columns;
(ii) columns co-occurring in the same table often share context beyond join relationships.
Thus, these hyperedges capture schema-level co-occurrence signals.

\label{sec:inter_hyperedges}
\noindent\textit{Step 3: Inter-Table Hyperedges.}
% {\color{blue}
To construct $\mathcal{E}_{\text{inter}}$, we build an auxiliary join graph $\mathcal{G}_{\text{join}}=(\mathcal{V}, \mathcal{P}_{\text{join}})$, where $(v_i, v_j)\in \mathcal{P}_{\text{join}}$ indicates that the corresponding column pair is joinable.
We apply Union-Find to obtain connected components $\{\mathcal{C}_k\}$ of $\mathcal{G}_{\text{join}}$;
each component with $|\mathcal{C}_k|\ge 2$ induces an inter-table hyperedge
$e^{\text{inter}}_k=\mathcal{C}_k$, and we set
$\mathcal{E}_{\text{inter}}=\{e^{\text{inter}}_k\}$.
To enrich the joinability signals without ground-truth join information, we generate the join graph with LLM-generated semantic variants as follows:
(i) \emph{key column selection.} We identify candidate join-key columns using standard keyness heuristics (e.g., completeness/uniqueness signals);
(ii) \emph{LLM augmentation.} For each selected join-key column $C_{\text{key}}$, we use an LLM to generate a semantically-preserving name variant $C'_{\text{key}}$ conditioned on table context (e.g., table/column names and sample rows).
The prompt template is shown in Box~\ref{box:prompt};
(iii) \emph{inter-table hyperedge generation.} We add edges $(C_{\text{key}}, C'_{\text{key}})$ to $\mathcal{G}_{\text{join}}$, then recompute connected components to form the final $\mathcal{E}_{\text{inter}}$.

\vspace{-2pt}
\begin{tcolorbox}[
      colback=gray!5!white,
      colframe=gray!75!black,
      title=Prompt Design,
      fonttitle=\bfseries
]
\label{box:prompt}
\vspace{-2pt}
\textbf{Input context.} Column name, table name, all column names in the table, and 3--5 sampled rows.

\textbf{Transformation guidance.} Generate a semantically equivalent variant that follows common database naming conventions:
\begin{itemize}[leftmargin=10pt, topsep=2pt, itemsep=0pt]
\item Separator changes (e.g., \texttt{CustomerID} $\rightarrow$ \texttt{Customer\_ID}, \texttt{customer-id})
\item Abbreviations (e.g., \texttt{CustomerID} $\rightarrow$ \texttt{CustID})
\item Case conventions (e.g., \texttt{CustomerID} $\rightarrow$ \texttt{customerId}, \texttt{CUSTOMERID})
\item Synonyms (e.g., \texttt{CustomerID} $\rightarrow$ \texttt{ClientID})
\end{itemize}

\textbf{Output.} Return a single most plausible variant in a structured format (e.g., JSON:
\texttt{\{"perturbed":"Cust\_ID"\}}).
\vspace{-4pt}
\end{tcolorbox}
\vspace{-2pt}

% {\color{blue}
\noindent\textbf{Theoretical Analysis.}
We provide a formal justification for why incorporating both intra-table and inter-table contexts
can improve joinable table discovery, and show that prior SOTA methods are special cases of our formulation. All the proofs are in our online full version~\cite{hyperjoin_full}.

% {\color{red} TODO: revise, $\mathcal{C}$ has already been used for column set, but the following is used for context $\mathcal{C}^{\text{inter}}(v_i)$ }
\vspace{-4pt}
\paragraph{Setup.}
Let $Y_{ij}\in\{0,1\}$ denote whether two columns $(v_i,v_j)$ are joinable.
Let $\mathbf{x}_i$ be the initial feature of column $v_i$.
For each column $v_i$, define two context variables induced by the constructed hypergraph:
(i) its intra-table context $\mathcal{N}^{\text{intra}}(v_i)$ (the set of columns co-located with $v_i$ in the same table),
and (ii) its inter-table context $\mathcal{N}^{\text{inter}}(v_i)$ (the set of columns connected with $v_i$ through joinable
connectivity, e.g., within the same join-connected component).
We write the available information for deciding joinability as a feature set:
\begin{equation}
\begin{aligned}
\mathcal{F}_{\text{single}}(i,j) &= \{\mathbf{x}_i,\mathbf{x}_j\},\\
\mathcal{F}_{\text{multi}}(i,j) &=
\Big\{\mathbf{x}_i,\mathbf{x}_j,\,
\mathcal{N}^{\text{intra}}(v_i),\mathcal{N}^{\text{intra}}(v_j),\\
&\hspace{18pt}
\mathcal{N}^{\text{inter}}(v_i),\mathcal{N}^{\text{inter}}(v_j)\Big\}.
\end{aligned}
\vspace{-4pt}
\end{equation}

Let $\Psi^*(\mathcal{F})$ denote the minimum achievable expected discovery risk under feature set $\mathcal{F}$:
\vspace{-2pt}
\[
\Psi^*(\mathcal{F}) \;=\; \inf_{g}\; \mathbb{E}\big[\ell(g(\mathcal{F}),Y_{ij})\big],
\]
% \vspace{-4pt}
where $g$ is any measurable scoring function and $\ell(\cdot)$ is a bounded loss (e.g., logistic or 0-1 loss).

% \vspace{-2pt}
% \begin{theorem}
\begin{restatable}{theorem}{bayesrisk}
\label{thm:bayes_risk}
For the two feature sets $\mathcal{F}_{\text{multi}}$ and $\mathcal{F}_{\text{single}}$ defined above,
the Bayes-optimal risk satisfies
\[
\Psi^*(\mathcal{F}_{\text{multi}})\ \le\ \Psi^*(\mathcal{F}_{\text{single}}).
\]
Moreover, if the added contexts are informative in the sense that
$\mathbb{P}(Y_{ij}=1\mid \mathcal{F}_{\text{multi}})\neq \mathbb{P}(Y_{ij}=1\mid \mathcal{F}_{\text{single}})$ with non-zero probability,
then the inequality is strict.
\end{restatable}
% \end{theorem}
\vspace{-2pt}

\vspace{-2pt}
% \begin{theorem}
\begin{restatable}{theorem}{priorspecial}
\label{thm:prior_special}
Consider our framework that learns column embeddings via a model $f_\theta$ operating on the constructed hypergraph.
DeepJoin and Snoopy can be viewed as special cases obtained by restricting the utilized context to
$\mathcal{F}_{\text{single}}$ (i.e., column-wise encoding without using $\mathcal{N}^{\text{intra}}$ or $\mathcal{N}^{\text{inter}}$).
% {\color{red}TODO: concise, "by choosing the hyperedge set as singleton hyperedges" is the practice of proof rather than the theorem. }
% \end{theorem}
\end{restatable}
\vspace{-2pt}
\subsection{\modelname~Model Design}
\label{sec:model_design}

\label{sec:hypergraph_network}
\label{sec:sf_hgn}

We present our \modelname, which operates on the constructed hypergraph to learn expressive column representations.
Our architecture consists of three components: (1) \textit{Positional Encoding}, (2) \textit{Local Hyperedge Aggregation}, and (3) \textit{Global Hyperedge Mixing}. 

% \subsubsection{Positional Encoding}

\label{sec:positional_encoding}

\vspace{2pt}
\noindent\textbf{Positional Encoding.}
% \noindent\textbf{Motivation.}
Standard message-passing graph neural networks (GNNs) suffer from position agnosticism~\cite{he2023generalization}, where nodes with identical local neighborhoods may receive identical representations even if they play different global roles in the joinability graph.
This limitation is formally characterized by the \emph{1-dimensional Weisfeiler--Lehman (1-WL) test}~\cite{weisfeiler1968reduction}.
We inject two types of positional encodings to mitigate this issue: (i) a \emph{table-level positional encoding} and (ii) a \emph{column-level Laplacian positional encoding}.

% \noindent Table-level positional encoding.
% \noindent {\color{red}The following titles are not clearly highlighted.}
% {\color{red}I does not get the meaning of "embedding table".}

\noindent \textit{Table-level positional encoding.}
We maintain a learnable lookup matrix $\mathbf{E}_{\text{tbl}} \in \mathbb{R}^{|\mathcal{T}| \times d}$, where each row stores a trainable $d$-dimensional vector for one table ID.
For a column $v$ belonging to table $t(v)$, we inject table identity by
\begin{equation}
\label{eq:table_pe}
\mathbf{h}_v^{\text{tbl-pe}} = \mathbf{x}_v^v + \alpha \cdot \mathbf{E}_{\text{tbl}}[t(v)],
\end{equation}
% where $\alpha$ is a learnable scalar weight.
% {\color{blue}
where $\alpha$ is a learnable scalar weight (initialized as $0.1$) and optimized jointly with the model.

\noindent \textit{Column-level positional encoding.}
% {\color{blue}
We adopt Laplacian positional encoding to capture global structural roles of columns in the joinability topology.
We construct a pairwise column graph $\mathcal{G}_{\text{col}}=(\mathcal{V},\mathcal{P}_{\text{join}})$, where each edge $(v_i,v_j)\in\mathcal{P}_{\text{join}}$ indicates that the corresponding column pair is joinable.
We build the symmetric adjacency matrix $\mathbf{A}\in\{0,1\}^{N\times N}$ by setting $A_{ij}=A_{ji}=1$ iff $(v_i,v_j)\in\mathcal{P}_{\text{join}}$, and $0$ otherwise.
We compute the normalized Laplacian
\begin{equation}
\label{eq:laplacian}
    \mathbf{L}_{\text{norm}} = \mathbf{I}_N - \mathbf{D}^{-1/2} \mathbf{A} \mathbf{D}^{-1/2},
\end{equation}
where $\mathbf{I}_N$ is the $N\times N$ identity matrix and $\mathbf{D}$ is the diagonal degree matrix.
We extract the smallest $k$ eigenvectors $\{\mathbf{v}_1,\ldots,\mathbf{v}_k\}$ and project them through a two-layer MLP:
\begin{equation}
\label{eq:lap_pe_mlp}
    \mathbf{PE}_{\text{col}}(v_i) =
    \mathbf{W}_2 \cdot \sigma\!\Big(\mathbf{W}_1 \cdot [\mathbf{v}_1[i],\ldots,\mathbf{v}_k[i]]^\top\Big).
\end{equation}
where $\mathbf{W}_1 \in \mathbb{R}^{d/2 \times k}$ and $\mathbf{W}_2 \in \mathbb{R}^{d \times d/2}$ are learnable projection matrices and $\sigma(\cdot)$ is a non-linear activation (e.g., ReLU).
This encoding captures global spectral structure of the joinability graph.
For efficiency, we precompute the eigenvectors offline and cache $\mathbf{PE}_{\text{col}}$ for lookup during training.
% }

\noindent \textit{Unified positional encoding.}
The final positional encoding is:
\begin{equation}
\label{eq:unified_pe}
\mathbf{h}_v = \mathbf{x}_v^v + \alpha \cdot \mathbf{E}_{\text{tbl}}[t(v)] + \beta \cdot \mathbf{PE}_{\text{col}}(v),
\end{equation}
where $\mathbf{x}_v^v$ is the raw initial column embedding,
$\mathbf{E}_{\text{tbl}}[t(v)]$ is the table-level embedding for the table containing column/node $v$, $\mathbf{PE}_{\text{col}}(v)$ is the Laplacian-based column positional encoding
and $\beta$ is a learnable scalar weight initialized identically to $\alpha$.
Together, they allow the model to automatically adjust the strength of each positional signal.
% \subsubsection{Local Hyperedge Aggregation}

\label{sec:hyperedge_gnn}
\noindent\textbf{Local Hyperedge Aggregation.}
Traditional GNNs suffer from over-squashing~\cite{deac2022expander,arnaiz2022diffwire}, where information from distant nodes gets exponentially compressed through multiple layers.
We adopt a hierarchical hyperedge-based approach to aggregate both the local and global information: by first aggregating columns into $M$ hyperedges ($M \ll N$) and performing message passing at the hyperedge level, we shorten path lengths and alleviate information bottlenecks.

\noindent \textit{Stage 1: Node-level feature transformation.}
Given PE-enhanced embeddings $\{\mathbf{h}_v^{(0)}\}_{v \in \mathcal{V}}$, we apply $L=2$ layers of node transformation:
\begin{equation}
\label{eq:gnn_layer}
    \mathbf{h}_v^{(\ell)} = \text{LayerNorm}(\sigma(\mathbf{W}^{(\ell)} \mathbf{h}_v^{(\ell-1)} + \mathbf{b}^{(\ell)})),
\end{equation}
where $\sigma(\cdot)$ is ReLU activation and $\mathbf{h}_v^{(0)} = \mathbf{h}_v$.
Unlike traditional GNN layers that aggregate neighbor information at this stage, we defer aggregation to Stage 2 at the hyperedge level.

\noindent \textit{Stage 2: Aggregation to hyperedge level.}
We aggregate node embeddings to the hyperedge level via mean pooling:
% \begin{equation}
% \label{eq:hyperedge_aggregation}
%     \mathbf{e}_j = \frac{1}{|e_j|} \sum_{v \in e_j} \mathbf{h}_v^{(L)},
% \end{equation}
% {\color{blue}
\begin{equation}
\label{eq:hyperedge_aggregation}
    \mathbf{x}_{e_j}^e = \frac{1}{|{e}_j|} \sum_{v \in {e}_j} \mathbf{h}_v^{(L)},
\end{equation}
where ${e}_j \subseteq \mathcal{V}$ denotes the node set of the $j$-th hyperedge.
Let $\mathbf{H}\in\mathbb{R}^{N\times d}$ stack node embeddings row-wise, i.e., $\mathbf{H}[v,:]=\mathbf{h}_v^{(L)}$,
and let $\mathbf{X}^e\in\mathbb{R}^{M\times d}$ stack hyperedge embeddings row-wise, i.e., $\mathbf{X}^e[j,:]=\mathbf{x}_{e_j}^e$.
This is efficiently computed via sparse matrix multiplication:
\begin{equation}
\mathbf{X}^e = \mathbf{D}_e^{-1}\boldsymbol{\Pi}^\top \mathbf{H},
\end{equation}
where $\boldsymbol{\Pi}\in\{0,1\}^{N\times M}$ is the node--hyperedge incidence matrix and
$\mathbf{D}_e \in \mathbb{R}^{M\times M}$ is diagonal with $(\mathbf{D}_e)_{jj}=\sum_{i=1}^{N} \boldsymbol{\Pi}_{ij}=|{e}_j|$.
Each node belongs to at most two hyperedges (one intra-table and at most one inter-table), making $\boldsymbol{\Pi}$ very sparse with time complexity
$O(\|\boldsymbol{\Pi}\|_0 \cdot d)$ where $\|\boldsymbol{\Pi}\|_0 \leq 2N$.
% }

% {\color{blue}
We denote the aggregated hyperedge embeddings as $\mathbf{Z}^{(0)}=\mathbf{X}^e$ and partition them by hyperedge type.
% }
Inter-table and intra-table hyperedges capture fundamentally different relationships: cross-table join semantics versus intra-table schema semantics.
We apply separate linear transformations to learn domain-specific patterns (lowercase = single embedding, uppercase = stacked matrix):
% {\color{blue}
\begin{align}
    \tilde{\mathbf{Z}}_{\text{inter}} &= \mathbf{Z}_{\text{inter}} \mathbf{W}_{\text{inter}}, \label{eq:transform_type1} \\
    \tilde{\mathbf{Z}}_{\text{intra}} &= \mathbf{Z}_{\text{intra}} \mathbf{W}_{\text{intra}}, \label{eq:transform_type2}
\end{align}
and concatenate them: $\tilde{\mathbf{Z}} = [\tilde{\mathbf{Z}}_{\text{inter}}; \tilde{\mathbf{Z}}_{\text{intra}}] \in \mathbb{R}^{M \times d}$.
% }
% \subsubsection{Global Hyperedge Mixing}

\label{sec:mixer}
\noindent\textbf{Global Hyperedge Mixing.}
Although local message passing on the hyperedge graph captures short-range dependencies, we still need global information exchange across hyperedges to model long-range join relationships.
However, standard hypergraph message passing is inherently local (typically 1-hop per layer) and thus requires multiple layers to achieve global coverage, which increases computation and can exacerbate over-squashing.
% }{\color{red} TODO: BiHMP?}
We design a global mixing mechanism that achieves all-to-all communication in a single layer, inspired by MLP-Mixer~\cite{tolstikhin2021mlp}, while preserving hypergraph structural priors through structure-aware attention.

% \noindent {\color{red}TODO: revise token and channel expression and concise.}

\noindent \textbf{Mixer Architecture.}
Our mixer consists of two sequential operations per layer: \HIname~(\HIabbr) across hyperedges and \HRname~(\HRabbr) within each hyperedge embedding.

For \HIabbr, we use structure-aware multi-head self-attention that combines content similarity with structural proximity.
To capture structural relationships, we first construct a hyperedge adjacency matrix $\mathbf{A}_{\text{hyperedge}} \in \mathbb{R}^{M \times M}$ where entry $[i,j]$ reflects the number of shared columns between hyperedges $i$ and $j$:
\begin{equation}
\label{eq:hyperedge_adjacency}
    % \mathbf{A}_{\text{hyperedge}} = \text{RowNormalize}(\boldsymbol{\Pi}^T \boldsymbol{\Pi}) - \text{diag}(\cdot).
    \mathbf{A}_{\text{hyperedge}} = \text{RowNormalize}\!\left(\boldsymbol{\Pi}^T \boldsymbol{\Pi} - \text{diag}(\boldsymbol{\Pi}^T \boldsymbol{\Pi})\right).
\end{equation}
We inject this structural prior into multi-head self-attention via an additive bias term:
\begin{equation}
\label{eq:structure_aware_attention}
    \text{head}_i = \text{softmax}\left(\frac{\mathbf{Q}_i \mathbf{K}_i^T}{\sqrt{d_k}} + \lambda \cdot \mathbf{A}_{\text{hyperedge}}\right) \mathbf{V}_i,
\end{equation}
where $\mathbf{Q}_i, \mathbf{K}_i, \mathbf{V}_i$ are query, key, value matrices for head $i$, and $\lambda$ is a learnable scalar initialized to 0.5.
This allows hyperedges with high structural connection but low content similarity to still exchange information.
We use $h=8$ heads and combine outputs with residual connection:
% $\tilde{\mathbf{E}}_{\text{\HIabbr}} = \tilde{\mathbf{E}} + \text{Concat}(\text{head}_1, \ldots, \text{head}_h) \mathbf{W}_O$
% {\color{blue}
$\mathbf{Z}_{\text{\HIabbr}} = \tilde{\mathbf{Z}} + \text{Concat}(\text{head}_1, \ldots, \text{head}_h) \mathbf{W}_O$.
% }

For \HRabbr, we apply a two-layer MLP:
% {\color{blue}
\begin{equation}
\label{eq:channel_mixer}
    {\mathbf{Z}}_{\text{out}} = {\mathbf{Z}}_{\text{\HIabbr}} + \mathbf{W}_2 \cdot \text{GELU}(\mathbf{W}_1 \cdot \text{LayerNorm}({\mathbf{Z}}_{\text{\HIabbr}})^T)^T,
    % \tilde{\mathbf{E}}_{\text{out}} = \tilde{\mathbf{E}}_{\text{\HIabbr}} + \mathbf{W}_2 \cdot \text{GELU}(\mathbf{W}_1 \cdot \text{LayerNorm}(\tilde{\mathbf{E}}_{\text{\HIabbr}})^T)^T,
\end{equation}
% }
where $\mathbf{W}_1 \in \mathbb{R}^{4d \times d}$ and $\mathbf{W}_2 \in \mathbb{R}^{d \times 4d}$ expand and project back with factor 4.
We stack $L_{\text{mixer}}=2$ layers, with $\mathbf{A}_{\text{hyperedge}}$ shared across all the layers.

\noindent \textbf{Propagating Information Back to Columns.}
We propagate enriched hyperedge embeddings back to columns:
% {\color{blue}
\begin{equation}
\label{eq:hyperedge_to_column}
    \mathbf{h}_v^{\text{hyperedge}} = \frac{1}{|\mathcal{N}(v)|} \sum_{j \in \mathcal{N}(v)} \mathbf{W}_{\text{h2c}} \tilde{\mathbf{z}_{e_j}^{(L_{\text{mixer}})}},
\end{equation}
% }
where $\mathcal{N}(v) = \{j \mid \boldsymbol{\Pi}_{vj} = 1\}$ is the set of hyperedges containing column $v$, and $\mathbf{W}_{\text{h2c}} \in \mathbb{R}^{d \times d}$ adapts hyperedge embeddings to the column-level space.
The final column embedding combines original and hyperedge-enhanced information:
\begin{equation}
\label{eq_final_column_embeddings}
    \mathbf{h}_v^{\text{final}} = \text{L2Norm}(\text{LayerNorm}(\mathbf{h}_v + \mathbf{h}_v^{\text{hyperedge}})).
\end{equation}
The residual connection preserves original semantics while L2 normalization enables efficient cosine similarity search.

The overall forward process of \modelname{} is summarized in Algorithm~\ref{algo:sf_hgn} in the appendix of the online full version~\cite{hyperjoin_full}. The following theorem shows that our model is more expressive than 1-WL test. The proof is in the online full version~\cite{hyperjoin_full}.

\vspace{-2pt}
% \begin{theorem}
\begin{restatable}[\modelname\ is strictly more expressive than 1-WL message passing]{theorem}{expressiveness}
\label{thm:expressiveness}
Let $\mathcal{H}_{\text{base}}$ be the hypothesis class of permutation-equivariant
message-passing encoders operating on the joinability graph with raw node features
$\mathbf{h}^{(0)}$ (i.e., 1-WL-initialized inputs).
Let $\mathcal{H}_{\text{pe}}$ be the class obtained by augmenting inputs with
\begin{equation}
\label{eq:pe_def}
    \text{PE}(v)=\alpha\mathbf{E}_{\text{tbl}}[t(v)]+\beta\mathbf{PE}_{\text{col}}(v),
\end{equation}
and using an injective node-wise update (as in \modelname).
Assume there exist two nodes $u$ and $v$ such that $\mathbf{h}^{(0)}_u=\mathbf{h}^{(0)}_v$ and $u,v$ are 1-WL-indistinguishable under the raw initialization, but $\text{PE}(u)\neq \text{PE}(v)$.
Then $\mathcal{H}_{\text{base}} \subsetneq \mathcal{H}_{\text{pe}}$.
% \end{theorem}
\end{restatable}
\vspace{-2pt}

% {\color{blue}
\noindent\textbf{Complexity Analysis.}
% Algorithm~\ref{algo:sf_hgn} consists of one-time offline preprocessing and online computation.
In Stage~1, computing the top-$k$ Laplacian eigenvectors is done offline (dense $O(N^3)$ for small graphs, or sparse iterative solvers for large/sparse graphs); online positional injection costs $O(Nd)$.
In Stage~2, the $L$-layer node-wise transformations cost $O(LNd^2)$, and pooling nodes to hyperedges via $\mathbf{X}^e=\mathbf{D}_e^{-1}\boldsymbol{\Pi}^\top \mathbf{H}$ costs $O(\|\boldsymbol{\Pi}\|_0 d)$.
In Stage~3, each mixer layer costs $O(M^2 d + M d^2)$, yielding $O(L_{\text{mixer}}(M^2 d + M d^2))$ in total.
Thus, the overall online complexity is
$O(LNd^2 + L_{\text{mixer}}(M^2 d + M d^2))$, with $\|\boldsymbol{\Pi}\|_0 \le 2N$ and typically $M \ll N$ (e.g., $M \approx 0.1N$--$0.3N$).
% }
% \subsubsection{Training Objectives}

\label{sec:training_objectives}
\noindent\textbf{Training Objectives.}
We adopt a self-supervised training strategy without labels by splitting each table into two overlapping subsets, treating join-key columns as positive pairs (including LLM-generated name variants), and sampling negatives from non-key columns within the same split and across different tables.
% }
Our training objective encourages the model to assign higher similarity to joinable column pairs than to non-joinable pairs by a margin.
Given final column embeddings $\mathbf{h}_{C}^{\mathrm{final}}$ produced by Algorithm~\ref{algo:sf_hgn}, we define the cosine similarity between columns $C_a$ and $C_b$ as
\begin{equation}
s(C_a, C_b) \;=\;
\frac{\mathbf{h}_{C_a}^{\mathrm{final}\,\top}\mathbf{h}_{C_b}^{\mathrm{final}}}
{\left\lVert \mathbf{h}_{C_a}^{\mathrm{final}} \right\rVert_2 \left\lVert \mathbf{h}_{C_b}^{\mathrm{final}} \right\rVert_2 }.
\end{equation}

Let $\mathcal{D}_{\mathrm{train}}^{+}$ denote the set of joinable (positive) pairs and $\mathcal{D}_{\mathrm{train}}^{-}$ denote the set of non-joinable (negative) pairs.
For each anchor column $C_i$, we sample a positive $C_i^{+}$ with $(C_i, C_i^{+})\in\mathcal{D}_{\mathrm{train}}^{+}$ and a negative $C_i^{-}$ with $(C_i, C_i^{-})\in\mathcal{D}_{\mathrm{train}}^{-}$.
We optimize the margin-based triplet ranking loss
\begin{equation}
\label{eq:triplet_loss}
\mathcal{L}_{\mathrm{triplet}}
\;=\;
\frac{1}{|\mathcal{D}_{\mathrm{train}}^{+}|}
\sum_{(C_i, C_i^{+}) \in \mathcal{D}_{\mathrm{train}}^{+}}
\Bigl[\, m + s(C_i, C_i^{-}) - s(C_i, C_i^{+}) \,\Bigr]_{+},
\end{equation}
where $[\cdot]_{+}=\max(\cdot,0)$ and $m$ is the margin hyperparameter.

% Section 5: online
\section{Online Ranking}
\label{sec:search}

We present the online ranking phase of \frameworkname, which efficiently retrieves coherent joinable columns from large-scale data lakes.
Section~\ref{sec:problem_formulation} discusses the 
% coherence-aware top-$K$
objective and its computational complexity.
Then, Section~\ref{sec:solution} presents an efficient global coherent reranking module: we first retrieve a manageable candidate pool via similarity search and construct a candidate graph, and then perform reranking using a greedy algorithm with MST. 
% This design achieves scalability by narrowing the search space and improves result quality by explicitly promoting coherence.

\subsection{Hardness Analysis}
\label{sec:problem_formulation}

Given a query column $C_q$, an initial candidate pool $\mathcal{C}_B$ returned by a fast similarity search stage, and a target size $K$, our goal is to select a subset $\mathcal{R}^* \subseteq \mathcal{C}_B$ with $|\mathcal{R}^*|=K$ that is both highly relevant to $C_q$ and mutually joinable as a coherent result set.
To quantify relevance and coherence, we use a pairwise affinity score $w(\cdot,\cdot)$ over columns, where $w(C_i, C_j) \in \mathbb{R}^+$ reflects their joinability and semantic similarity.
A coherent result set should be both relevant to the query and well-connected internally.
A natural way to capture this trade-off is to score a size-$K$ set by the following objective:

\begin{equation}
\label{eq:mst_objective}
\mathcal{R}^* = \underset{\mathcal{R} \subseteq \mathcal{C}_B, |\mathcal{R}|=K}{\arg\max} \ G(\mathcal{R}),
\end{equation}
where the objective function $G(\mathcal{R})$ is:
\begin{equation}
\label{eq:objective_function}
    G(\mathcal{R}) = {\sum_{C \in \mathcal{R}} w(C_q, C)} + \lambda \cdot {\text{Coherence}(\mathcal{R})}
\end{equation}
The first term aggregates individual query relevance, while the second term measures how tightly the selected columns are interconnected.
The hyperparameter $\lambda \geq 0$ controls the trade-off between relevance and coherence.

Selecting the optimal $K$-column subset under Eq.~\ref{eq:mst_objective} is a combinatorial optimization problem.
The following theorem presents the hardness analysis for solving such problem. 
The proof is in the online full version~\cite{hyperjoin_full}.

\vspace{-2pt}
\begin{restatable}[NP-hardness of Coherence-Aware Selection]{theorem}{nphard}
\label{thm:np_hard_formulation}
Maximizing the objective $G(\mathcal{R})$ in Eq.~\ref{eq:mst_objective} under the cardinality constraint $|\mathcal{R}|=K$ is NP-hard when $\text{Coherence}(\mathcal{R})$ is instantiated as the weight of a maximum spanning tree rooted at $C_q$ on the induced subgraph over $\{C_q\}\cup\mathcal{R}$.
\end{restatable}
% }
\vspace{-4pt}
\subsection{Global Coherent Search}
\label{sec:solution}

Given a query column $C_q$, the offline phase provides a structure-aware embedding $h_C$ for every column $C$ in the data lake.
At query time, we aim to solve the coherence-aware top-$K$ selection objective in Eq.~\ref{eq:mst_objective}.
% While the coherence score of a \emph{fixed} set can be computed efficiently, 
Theorem~\ref{thm:np_hard_formulation} shows that selecting the best size-$K$ subset is NP-hard.
We therefore adopt a two-stage framework that balances scalability and effectiveness:(1) retrieve a manageable candidate pool and build a weighted candidate graph; and (2) perform MST-based greedy reranking to return a globally coherent top-$K$ set.

\noindent\textbf{Candidate Graph Construction.}
\label{sec:candidate_graph}
Given a query column $C_q$, the data lake contains $N$ total columns.
To avoid the intractable search space $\binom{N}{K}$, we first retrieve a manageable candidate pool of size $B$ (where $K < B \ll N$) via efficient similarity search, and then perform coherence-aware optimization only within this pool, yielding a candidate set $\mathcal{C}_B$.

\begin{definition}[Candidate Graph]
\label{def:candidate_graph}
Given a query column $C_q$ and a candidate set $\mathcal{C}_B$ obtained via similarity search (typically $|\mathcal{C}_B| = B > K$), the candidate graph is defined as a weighted undirected graph $\mathcal{G}_C = (\mathcal{V}_C, \mathcal{E}_C, w)$ where $\mathcal{V}_C = \{C_q\} \cup \mathcal{C}_B$ contains the query column and all candidate columns, and the edge set $\mathcal{E}_C = \mathcal{E}_{query} \cup \mathcal{E}_{cand}$ consists of query-to-candidate edges and candidate-to-candidate edges.
The weight function $w: \mathcal{E}_C \to \mathbb{R}^+$ assigns each edge a positive weight:
\begin{equation}
    w(C_i, C_j) = \text{sim}(h_{C_i}, h_{C_j}),
\end{equation}
where $h_{C_i}, h_{C_j}$ are the structure-aware column embeddings learned by \modelname~(Section~\ref{sec:sf_hgn}).
\end{definition}

The candidate graph is constructed through the following steps.
First, for a given query column $C_q$, we perform efficient similarity search over all columns in the data lake based on cosine similarity between their learned embeddings.
This yields a ranked list of the Top-$B$ most relevant candidates, $\mathcal{C}_B = \{C_1, C_2, \ldots, C_B\}$, ensuring high recall.
Second, we establish query-to-candidate edges $\mathcal{E}_{query} = \{(C_q, C) \mid C \in \mathcal{C}_B\}$ connecting the query column to all candidates, with edge weights $w(C_q, C_i) = \text{sim}(h_{C_q}, h_{C_i})$ reflecting the relevance of each candidate to the query.
Third, to capture the inter-candidate coherence, we establish candidate-to-candidate edges $\mathcal{E}_{cand} = \{(C_i, C_j) \mid C_i, C_j \in \mathcal{C}_B, i \neq j\}$ between all pairs of candidates, with edge weights $w(C_i, C_j) = \text{sim}(h_{C_i}, h_{C_j})$ reflecting the joinability and semantic similarity between the two candidates.
This stage reduces the search space from $\binom{N}{K}$ to $\binom{B}{K}$ by restricting selection to the top-$B$ candidate while preserving high recall.

\noindent\textbf{MST-based Greedy Reranking.}
\label{sec:mst_greedy_reranking}
With the candidate graph $\mathcal{G}_C$ built, the core challenge is to efficiently select a set $\mathcal{R} \subseteq \mathcal{C}_B$ that maximizes the objective $G(\mathcal{R})$ in Eq.~\ref{eq:mst_objective}.
This requires instantiating the coherence term, $\text{Coherence}(\mathcal{R})$, which quantifies how well the selected columns can be integrated for joint analysis.
A seemingly straightforward way to enforce coherence is to require the returned columns to form a dense clique, i.e., every pair is strongly joinable.
This assumption is overly restrictive for real data lakes where join graphs are typically sparse and multi-hop joins through bridge tables are common.

Instead, coherence should ensure that every column in $\mathcal{R}$ is reachable from the query $C_q$ through a sequence of strong join relationships.
This notion is naturally quantified by the weight of a \emph{maximum spanning tree} (MST) constructed on the induced subgraph over ${C_q} \cup \mathcal{R}$.
The MST extracts the set of strongest possible edges that connect all columns without cycles, ensuring reachability from the query anchor to every result while accommodates multi‑hop joins common in real data lakes, and avoids the over‑strictness of dense subgraph measures. 
Formally, for a vertex set $\mathcal{V}' = {C_q} \cup \mathcal{R}$, let $\mathcal{G}_C[\mathcal{V}']$ denote the subgraph induced by $\mathcal{V}'$ in $\mathcal{G}_C$.
The coherence of $\mathcal{R}$ is defined as the weight of a MST of $\mathcal{G}_C[\mathcal{V}']$:
\begin{equation}
\label{eq:coherence_def}
\text{Coherence}(\mathcal{R}) = \max_{T \in \mathcal{S}(\mathcal{V}')} \sum_{(u,v) \in E(T)} w(u, v),
\end{equation}
where $\mathcal{S}(\mathcal{V}')$ is the set of all spanning trees of $\mathcal{G}_C[\mathcal{V}']$, and $E(T)$ denotes the edge set of tree $T$.

% Directly optimizing $G(\mathcal{R})$ with this MST-based coherence is NP-hard (Theorem~\ref{thm:np_hard_formulation}).
Evaluating $G(\mathcal{R})$ by explicitly constructing an MST for each candidate subset $\mathcal{R}$ is computationally prohibitive for online search.
We therefore propose a greedy algorithm that cleverly combines the logic of Prim's MST construction with subset selection, avoiding the need to recompute an MST from scratch at each iteration.
The algorithm incrementally builds the result set $\mathcal{R}$ by simulating the growth of an MST \emph{rooted at the query $C_q$} as follows:
\begin{itemize}[leftmargin=10pt, topsep=1pt]
\item \textbf{Initialization:} Start with $\mathcal{R}=\emptyset$ and define the current connected component is $S=\{C_q\}$.
Since $S$ initially contains only the query node, the coherence gain term for any candidate $C$ is simply $w(C, C_q)$; hence the first selected node is the one with the highest query relevance.
\item \textbf{Selection rule:} At each step, for every candidate $C \in \mathcal{C}_B \setminus \mathcal{R}$, compute the marginal gain
\begin{equation}
\label{eq:marginal_gain}
\Delta(C \mid \mathcal{R}) = w(C_q, C) + \lambda \cdot \max_{C' \in S} w(C, C'),
\end{equation}
where $S = \{C_q\} \cup \mathcal{R}$. The first term captures the relevance to the query, while the second term measures the strongest possible connection between $C$ and the current component $S$—this mirrors the cut step in Prim’s algorithm, which attaches a new node via the heaviest edge crossing to the growing tree. The algorithm then selects the candidate
\[
    C^* = \arg\max_{C \in \mathcal{C}_B \setminus \mathcal{R}} \Delta(C \mid \mathcal{R}).
\]
\item \textbf{Maintenance:} Update the result set $\mathcal{R} \leftarrow \mathcal{R} \cup \{C^*\}$ and the connected component $S \leftarrow S \cup \{C^*\}$. This incremental expansion ensures that each newly added node is linked to the existing structure by the strongest available edge, thereby progressively building a high‑weight connectivity backbone. Repeat the selection and maintenance steps until $|\mathcal{R}| = K$.
\end{itemize}
This iterative process ensures that the final set $\mathcal{R}$ is connected to $C_q$ via a tree composed of the strongest edges chosen at each step. The sequence of $\max$ operations implicitly constructs a connectivity backbone rooted at $C_q$ that approximates the MST gain; the exact MST-based coherence of the final set can be obtained by running a MST algorithm on the induced subgraph $\mathcal{G}_C[\{C_q\}\cup\mathcal{R}]$ if needed, without requiring recomputation for every intermediate subset during the online search.

\begin{algorithm}[t]
\caption{Global Coherent Search}
\label{algo:coherence_aware_search}
\LinesNumbered
\DontPrintSemicolon
\KwIn{Query column $C_q$, Embedding DB $\{h_C \mid C \in \mathcal{C}\}$, Target size $K$, Candidate size $B$, Coherence weight $\lambda$.}
\KwOut{Final result set $\mathcal{R}$.}

\tcp{\textbf{Phase 1: Candidate Graph Construction}}
$\mathcal{C}_B \leftarrow \text{Top-$B$ similarity search using } h_{C_q}$ \;
Construct candidate graph $\mathcal{G}_C = (\{C_q\} \cup \mathcal{C}_B, E_C, w)$ \;
Compute edge weights $w(C_i, C_j) \leftarrow \text{sim}(h_{C_i}, h_{C_j})$ for edges \;

\tcp{\textbf{Phase 2: MST-based Greedy Reranking}}
$\mathcal{R} \leftarrow \emptyset$ \;
$U \leftarrow \mathcal{C}_B$ \tcp*{unselected candidates}
\ForEach{$C \in U$}{
    $\mathrm{rel}[C] \leftarrow w(C_q, C)$ \;
    $\mathrm{key}[C] \leftarrow w(C_q, C)$ \tcp*{best attachment to $S=\{C_q\}$}
}
\For{$i=1,\ldots,K$}{
    \tcp{Select candidate maximizing marginal gain}
    $C^* \leftarrow \arg\max_{C \in U} \big(\mathrm{rel}[C] + \lambda\cdot \mathrm{key}[C]\big)$ \;
    $\mathcal{R} \leftarrow \mathcal{R} \cup \{C^*\}$ \;
    $U \leftarrow U \setminus \{C^*\}$ \;
    \tcp{Update keys for remaining candidates}
    \ForEach{$C \in U$}{
        \If{$w(C, C^*) > \mathrm{key}[C]$}{
            $\mathrm{key}[C] \leftarrow w(C, C^*)$ \;
        }
    }
}
\Return $\mathcal{R}$ \;
\end{algorithm}

\noindent\textbf{Algorithm Description.}
The complete global coherent search procedure is summarized in Algorithm~\ref{algo:coherence_aware_search}, which integrates both candidate graph construction and MST-based reranking. 
Given a query column $C_q$, the algorithm first retrieves the Top-$B$ most relevant candidates via similarity search and constructs a weighted candidate graph $\mathcal{G}_C$ (lines 1–3). 
In the reranking phase (lines 4–16), the algorithm grows a result set $\mathcal{R}$ by iteratively selecting the candidate with the highest marginal gain, as defined in Eq.~\ref{eq:marginal_gain}. This gain combines the candidate's relevance to $C_q$ with its strongest connection to the current tree (including $C_q$ itself), thereby simulating a Prim-style expansion of a maximum spanning tree rooted at the query. The process continues until $K$ columns are selected, yielding a coherent result set optimized for both relevance and internal joinability.

\noindent\textbf{Approximation Guarantee.}
We provide a theoretical justification for the greedy reranking algorithm. 
We denote by $\Delta(C\mid\mathcal{R})$ the surrogate marginal gain defined in Eq.~\ref{eq:marginal_gain}.
We prove that the greedy rule in Eq.~\ref{eq:marginal_gain} optimizes a tractable surrogate whose marginal gain provably lower-bounds the true marginal gain contributed by the MST-based coherence term.
The proof is in the online full version~\cite{hyperjoin_full}.

\begin{restatable}[MST Gain Lower Bound]{lemma}{mstlowerbound}
\label{lem:mst_lower_bound}
For any $S\ni C_q$ and $v\notin S$, let $\text{MST}(S)$ be the MaxST weight on $\mathcal{G}_C[S]$. Then
\[
\text{MST}(S\cup\{v\})-\text{MST}(S)\ \ge\ \max_{u\in S} w(v,u).
\]
\end{restatable}

\vspace{-8pt}
\begin{restatable}[Surrogate Marginal Gain Property]{theorem}{surrogateguarantee}
\label{thm:surrogate_guarantee}
% Let $S=\{C_q\}\cup\mathcal{R}$ and let $C^*$ be selected by Eq.~\ref{eq:marginal_gain}.
Let $S=\{C_q\}\cup\mathcal{R}$ and let $C^*$ be the candidate selected by maximizing the surrogate marginal gain $\Delta(C\mid\mathcal{R})$ in Eq.~\ref{eq:marginal_gain}.
Then
\[
G(\mathcal{R}\cup\{C^*\})-G(\mathcal{R})\ \ge\ \Delta(C^*\mid\mathcal{R}).
\]
Moreover, if $u_t\in\arg\max_{u\in S_{t-1}} w(C^{(t)},u)$, then $\{(C^{(t)},u_t)\}_{t=1}^K$ forms a spanning tree over $\{C_q\}\cup\mathcal{R}$ and
\[
\text{Coherence}(\mathcal{R})=\text{MST}(\{C_q\}\cup\mathcal{R})\ \ge\ \sum_{t=1}^K w(C^{(t)},u_t).
\]
\end{restatable}

This analysis shows Eq.~\ref{eq:marginal_gain} is a conservative estimate of the true marginal gain in $G(\cdot)$: 
at each iteration, the objective increases by at least the computed surrogate gain. 
Moreover, the maintained attachment form a spanning tree over the final vertex set, whose total weight is lower-bounded by the MST-based coherence.

\noindent\textbf{Complexity Analysis.}
The per-query time complexity is dominated by two components: candidate retrieval and graph construction cost $O(Nd + N\log B + B^2 d)$, and greedy reranking cost $O(BK)$. Specifically, Stage~1 retrieves $B$ candidates via similarity search over $N$ columns ($O(Nd + N\log B)$) and constructs a candidate graph with $O(B^2)$ edges whose weights are computed in $O(B^2 d)$ time. Stage~2 maintains for each candidate its best connection (key) to the current tree; each of the $K$ iterations involves selecting the candidate with maximum marginal gain in $O(B)$ time and updating the keys of the remaining candidates in $O(B)$ time, yielding an $O(BK)$ total cost. The overall complexity is $O(Nd + N\log B + B^2 d + BK)$, and the space complexity is $O(B^2)$ to store the candidate graph's pairwise weights, which is efficient for typical parameters.

\section{Experimental Evaluation}
\label{sec:Experiment}

\begin{table}[t]
\centering
\caption{Statistics of benchmark datasets}
\vspace{-4pt}
\label{tab:datasets}
\begin{tabular}{|l|r|r|r|r|}
\hline
\cellcolor{lightgray}\textbf{Dataset} & \cellcolor{lightgray}\textbf{\#Tables} & \cellcolor{lightgray}\textbf{\#Columns} & \cellcolor{lightgray}\textbf{\#Queries} & \cellcolor{lightgray}\textbf{\#Join Pairs} \\ \hline \hline
USA        & 2,270   & 21,702  & 20 & 15,584 \\ \hline
CAN        & 9,590   & 70,982  & 30 & 57,124 \\ \hline
UK\_SG     & 1,166   & 11,293  & 20 & 8,164  \\ \hline
Webtable   & 2,922   & 21,787  & 30 & 21,029 \\ \hline
\end{tabular}
\vspace{-6pt}
\end{table}

\subsection{Experimental Setup}

\noindent\textbf{Datasets.}
We evaluate joinable table discovery on four real-world data lake benchmarks.
Three benchmarks are derived from government open data portals in the United States, Canada, and the United Kingdom/Singapore~\cite{opendata}; for convenience, we refer to them as \textsc{USA}, \textsc{CAN}, and \textsc{UK\_SG}, respectively.
We also use Webtable~\cite{webtable}, a benchmark derived from the WDC Web Table Corpus, and keep only the English relational web tables.
For each table, we extract the key column as specified by the benchmark metadata.
All benchmarks provide metadata such as table titles, column names, and cell context, and have been used in proir works \cite{bhagavatula2015tabel,dong2021efficient,wang2021retrieving,DBLP:conf/acl/YinNYR20,zhu2019josie,dong2023deepjoin,zhu2016lsh}.
Table~\ref{tab:datasets} reports their statistics.

\noindent\textbf{Baselines.} We compare the following methods.
(1) JOSIE~\cite{zhu2019josie}:
Token-based set similarity search using inverted index and posting lists with cost-model-based greedy merge algorithm.
(2) LSH Ensemble~\cite{zhu2016lsh}:
MinHash-based containment search using Locality Sensitive Hashing with dynamic partitioning.
(3) DeepJoin~\cite{dong2023deepjoin}: Pre-trained transformer model for column embedding.
(4) Snoopy~\cite{guo2025snoopy}: State-of-the-art column embedding model that uses proxy columns to improve semantic join discovery.
% contrastive learning.
(5) Omnimatch~\cite{koutras2025omnimatch}: Multi-relational graph neural network that combines multiple similarity signals for table matching and joinability discovery.

\noindent\textbf{Metrics.}
Following prior work~\cite{deng2024lakebench}, we evaluate retrieval quality using Precision@K and Recall@K.
Let $Q$ denote the set of query columns, $\mathcal{R}_K(q)$ denote the set of Top-$K$ retrieved columns for query $q$, and $\mathcal{R}^*(q)$ denote the set of ground-truth joinable columns for $q$.
We compute:
\begin{equation}
\label{eq:precision_at_k}
\mathrm{Precision@K}=\frac{1}{|Q|}\sum_{q\in Q}\frac{|\mathcal{R}_K(q)\cap \mathcal{R}^*(q)|}{K},
\end{equation}
\begin{equation}
\label{eq:recall_at_k}
\mathrm{Recall@K}=\frac{1}{|Q|}\sum_{q\in Q}\frac{|\mathcal{R}_K(q)\cap \mathcal{R}^*(q)|}{|\mathcal{R}^*(q)|}.
\end{equation}
We report results for $K \in \{5, 15, 25\}$ to assess performance at cutoffs.

% }
% We report results for $K \in \{5, 15, 25\}$ to evaluate ranking quality at different cutoffs.

\noindent\textbf{Implementation Details.}
We implement HyperJoin in PyTorch 2.0.1 with Python 3.9.
For initial column encoding, we use a text encoder with learnable token embeddings (vocabulary size $\sim$1500) to encode table names and column names, followed by mean pooling.
For column content features, we use pre-trained fastText embeddings to encode cell values, which are then aggregated and projected through a two-layer MLP with ReLU activation and dropout.
The model architecture consists of 2 hypergraph convolutional layers and 2 MLP-Mixer layers, with embedding dimension 512.
We use three-level positional encoding with dimension 16 for table-level and column-level structural information, computed via Laplacian eigenvector decomposition on the joinable graph.
For training, we employ the Adam optimizer with learning rate $4 \times 10^{-4}$, batch size 64, and train for 30 epochs.
We use triplet loss with margin 1.0 for contrastive learning, with hard negative mining enabled.
Dropout rate is set to 0.05 to prevent overfitting.
For self-supervised learning, we generate augmented columns using DeepSeek-V3 via API with temperature 0.7 and top-p sampling 0.9.
For global coherent reranking, we set candidate size $B=50$ and coherence weight $\lambda=1.0$.
All experiments are conducted on a workstation with an Intel Xeon E E-2488 CPU (8 cores, 16 threads), an NVIDIA A2 GPU with 15GB memory, and 128 GB RAM.

\begin{table*}[t]
\centering
\small
\caption{Performance comparison across four benchmark datasets.
% HyperJoin denotes supervised training; HyperJoin/LF denotes label-free training. All variants use Mixer + MST reranking. Bold indicates best performance per metric-dataset combination.
}
\vspace{-10pt}
\label{tab:main_results}
\setlength{\tabcolsep}{2.5pt}
\begin{tabular}{p{1.8cm}<{\centering} | p{0.48cm}<{\centering} p{0.48cm}<{\centering} p{0.48cm}<{\centering} p{0.48cm}<{\centering} p{0.48cm}<{\centering} p{0.48cm}<{\centering} | p{0.48cm}<{\centering} p{0.48cm}<{\centering} p{0.48cm}<{\centering} p{0.48cm}<{\centering} p{0.48cm}<{\centering} p{0.48cm}<{\centering} | p{0.48cm}<{\centering} p{0.48cm}<{\centering} p{0.48cm}<{\centering} p{0.48cm}<{\centering} p{0.48cm}<{\centering} p{0.48cm}<{\centering} | p{0.48cm}<{\centering} p{0.48cm}<{\centering} p{0.48cm}<{\centering} p{0.48cm}<{\centering} p{0.48cm}<{\centering} p{0.48cm}<{\centering}}
\toprule
\multirow{3}{*}{\textbf{Method}} 
& \multicolumn{6}{c|}{\textbf{Webtable}} 
& \multicolumn{6}{c|}{\textbf{USA}} 
& \multicolumn{6}{c|}{\textbf{CAN}} 
& \multicolumn{6}{c}{\textbf{UK\_SG}} \\
\cmidrule(lr){2-7} \cmidrule(lr){8-13} \cmidrule(lr){14-19} \cmidrule(lr){20-25}
& \multicolumn{3}{c}{Precision} & \multicolumn{3}{c|}{Recall} 
& \multicolumn{3}{c}{Precision} & \multicolumn{3}{c|}{Recall} 
& \multicolumn{3}{c}{Precision} & \multicolumn{3}{c|}{Recall} 
& \multicolumn{3}{c}{Precision} & \multicolumn{3}{c}{Recall} \\
\cmidrule(lr){2-4} \cmidrule(lr){5-7} 
\cmidrule(lr){8-10} \cmidrule(lr){11-13} 
\cmidrule(lr){14-16} \cmidrule(lr){17-19} 
\cmidrule(lr){20-22} \cmidrule(lr){23-25}
& @5 & @15 & @25 & @5 & @15 & @25 
& @5 & @15 & @25 & @5 & @15 & @25 
& @5 & @15 & @25 & @5 & @15 & @25 
& @5 & @15 & @25 & @5 & @15 & @25 \\
\midrule

JOSIE 
& 35.3 & 25.1 & 15.9 & 10.9 & 23.2 & 24.2
& 76.0 & 74.7 & 68.6 & 17.7 & 52.2 & 79.9
& 56.7 & 52.4 & 44.1 & 13.1 & 35.8 & 49.5
& 27.0 & 28.7 & 23.2 & 8.8 & 27.0 & 34.9 \\

LSH 
& 58.7 & 51.4 & 46.7 & 17.4 & 40.3 & 52.5
& 81.0 & 72.1 & 72.0 & 18.5 & 41.8 & 50.5
& 68.7 & 65.7 & 60.7 & 13.6 & 35.7 & 46.4
& 26.0 & 26.3 & 26.4 & 6.0 & 14.4 & 22.2 \\

DeepJoin 
& 69.3 & 55.8 & 45.9 & 22.1 & 51.0 & 68.1
& 87.0 & 68.7 & 48.8 & 20.2 & 48.1 & 57.1
& 64.0 & 53.3 & 42.4 & 14.7 & 36.8 & 47.6
& 66.0 & 55.0 & 42.2 & 19.6 & 48.5 & 60.2 \\

Snoopy 
& 71.3 & 55.3 & 46.3 & 22.8 & 50.0 & 68.3
& 86.0 & 74.7 & 63.8 & 19.9 & 51.9 & 74.0
& 79.3 & 65.1 & 58.7 & 18.0 & 45.0 & 66.7
& 72.0 & 60.3 & 44.0 & 20.8 & 53.5 & 64.0 \\

Omnimatch 
& 70.0 & 59.1 & 45.3 & 21.8 & 55.6 & 68.7
& 63.0 & 74.7 & 70.4 & 14.6 & 52.2 & 82.0
& 60.7 & 62.4 & 56.5 & 13.3 & 41.0 & 61.4
& 56.0 & 48.3 & 39.2 & 16.5 & 41.8 & 55.3 \\
\midrule

HyperJoin 
& 81.3 & 74.7 & 53.6 & 25.9 & 69.8 & 80.7
& 92.0 & 95.3 & 83.6 & 21.3 & 66.4 & 97.1
& 83.3 & 81.8 & 65.9 & 18.9 & 55.9 & 74.2
& 94.0 & 89.0 & 70.2 & 27.2 & 77.1 & 98.8 \\

\bottomrule
\end{tabular}
\end{table*}

\begin{figure*}[t]
\centering
\includegraphics[width=1\textwidth]{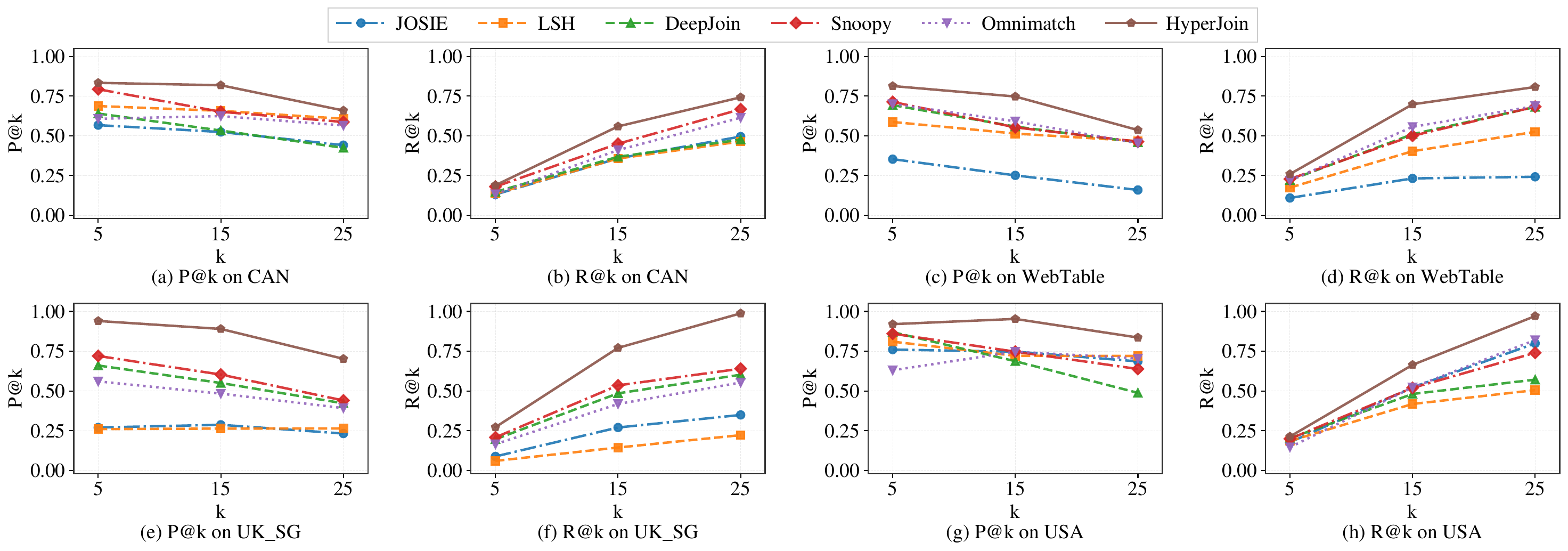}
\caption{Performance comparison across different datasets and k values.}
\label{fig:line_charts}
\end{figure*}

\subsection{Overall Evaluation}

\noindent\textbf{Exp-1: Overall Performance Comparison.}
Table~\ref{tab:main_results} and Figure~\ref{fig:line_charts} present the comprehensive comparison across all methods on four benchmark datasets.
Overall, HyperJoin achieves the best Precision@K and Recall@K across \textit{WebTable}, \textit{USA}, \textit{CAN} and \textit{UK\_SG}.
As shown in the table, LSH performs best among traditional baselines and PLM-based baseline generally demonstrate superior performance compared to traditional methods. 
Among the PLM-based models, Snoopy yields the best results. 
Notably, HyperJoin further outperforms Snoopy, achieving the highest overall accuracy.
Specifically, at \(K=15\), HyperJoin improves P@15 by 21.4 points and R@15 by 17.2 points on average compared to Snoopy.
The improvements are most pronounced on government-derived benchmarks, where HyperJoin exceeds Snoopy by margins of up to 28.7 points in P@15 (on UK\_SG).
This stems from the inherent properties of government data, which are characterized by strong structural and semantic regularities such as consistent column naming, rich metadata, and clear domain hierarchies.
HyperJoin's LLM-augmented hypergraph and hierarchical message passing are explicitly designed to capture these precise forms of contextual signals. 
In contrast, WebTable consists predominantly of numerical content where such semantic signals are inherently less frequent, yet HyperJoin still achieves substantial gains by effectively utilizing the available structural and metadata context.
These consistent margins demonstrate that HyperJoin provides not only higher accuracy but also stronger generalization across heterogeneous datasets.
Compared to traditional set-overlap methods (JOSIE, LSH Ensemble), HyperJoin demonstrates even larger improvements, with average R@25 gains exceeding 40\% across datasets.

We further analyze whether HyperJoin’s advantage is more evident when the number of discovered tables is limited (small \(K\)) or comparatively higher (large \(K\)).
At \(K=5\), HyperJoin achieves an average P@5 of 87.7\% and improves over Snoopy by 10.5 points in P@5 and 3.0 points in R@5, suggesting better early-stage ranking quality under limited discovery.
As \(K\) increases, the advantage becomes more pronounced, especially in recall.
At \(K=25\), HyperJoin exceeds Snoopy by 15.1 points in P@25 and 19.5 points in R@25 on average, indicating that it can discover substantially more joinable tables while maintaining strong precision.
This pattern is particularly clear on UK\_SG, where the recall margin grows from +6.4 points at \(K=5\) to +34.8 points at \(K=25\).
Overall, the results suggest that HyperJoin is competitive even at small \(K\), while its gains are larger at higher \(K\).
This shows jointly modeling relevance and coherence becomes increasingly beneficial as broader discovery makes relevance alone insufficient.

\begin{table}[t]
\centering
\caption{Ablation study using P@15 and R@15 (percentage). We quantify the effect of CR (Coherent Reranking) by comparing the full model with and without CR, and evaluate core module ablations by removing HIN or HG (Hypergraph).}
% Results are reported in percentage.
\label{tab:ablation_small}
\small
\setlength{\tabcolsep}{3.2pt}
\resizebox{\columnwidth}{!}{
\begin{tabular}{l | cc cc cc cc}
\toprule
\multirow{2}{*}{\textbf{Variant}} &
\multicolumn{2}{c}{\textbf{Webtable}} &
\multicolumn{2}{c}{\textbf{USA}} &
\multicolumn{2}{c}{\textbf{CAN}} &
\multicolumn{2}{c}{\textbf{UK\_SG}} \\
\cmidrule(lr){2-3}\cmidrule(lr){4-5}\cmidrule(lr){6-7}\cmidrule(lr){8-9}
& P@15 & R@15 & P@15 & R@15 & P@15 & R@15 & P@15 & R@15 \\
\midrule
Full  &
\textbf{74.67} & \textbf{69.75} &
\textbf{95.33} & \textbf{66.42} &
\textbf{81.78} & \textbf{55.89} &
\textbf{89.00} & \textbf{77.08} \\
w/o CR &
71.33 & 67.09 &
90.67 & 63.20 &
75.56 & 51.73 &
86.67 & 75.23 \\
w/o HIN &
51.11 & 46.88 &
75.67 & 52.83 &
60.44 & 41.54 &
69.67 & 58.73 \\
w/o HG &
55.11 & 49.75 &
72.33 & 50.34 & 64.89 & 45.01 &
59.67 & 53.19 \\
\bottomrule
\end{tabular}
}
\vspace{-4pt}
\end{table}

\noindent\textbf{Exp-2: Ablation Study.}
We conduct an ablation study to analyze the contribution of each component in \frameworkname. 
Table~\ref{tab:ablation_small} reports P@15 and R@15 on four datasets. 
We first examine coherent reranking (CR) by comparing the full model with and without CR. 
CR improves performance consistently across datasets, with average gains of +4.1 P@15.
The largest benefit is on CAN (P@15: 75.56$\rightarrow$81.78, +6.22; R@15: 51.73$\rightarrow$55.89, +4.16), while the improvement remains clear on USA (+4.66 P@15, +3.22 R@15), suggesting that set-level coherence suppresses candidates that are locally plausible but globally inconsistent.
We then remove key modules to understand where the performance comes from (under the same setting as w/o CR).
Removing HIN causes substantial and consistent degradations.
The most severe drop occurs on WebTable (P@15: 71.33$\rightarrow$51.11, -20.22; R@15: 67.09$\rightarrow$46.88, -20.21), highlighting the importance of fine-grained interaction modeling for filtering numerous superficially similar distractors in heterogeneous web tables.
Removing the hypergraph (HG) module shows a more dataset-dependent impact: the drop is particularly large on UK\_SG (P@15: 86.67$\rightarrow$59.67; R@15: 75.23$\rightarrow$53.19) and is also notable on USA (P@15: -18.34; R@15: -12.86), indicating that structural information is crucial when global regularities are strong.
Overall, the ablations demonstrate that HIN is a core contributor across datasets, HG is especially important on more structured corpora, and CR further complements both by improving set-level coherence.

\begin{figure}[t]
  \centering
  \captionsetup{skip=0pt}
  \includegraphics[width=0.45\textwidth]{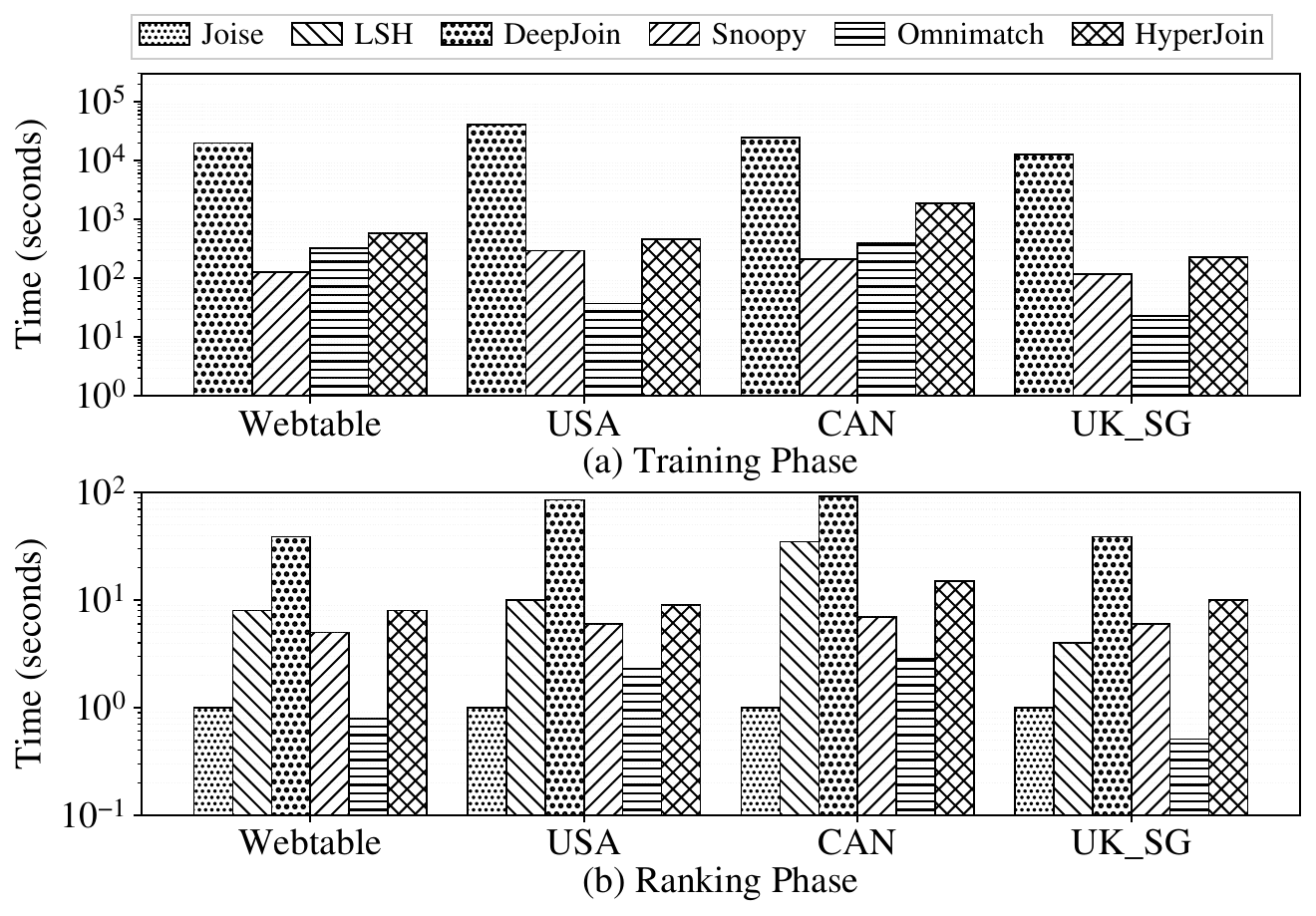}
  % \vspace{-1pt}
  \caption{Training and evaluation efficiency across datasets.}
  \vspace{-4pt}
  \label{fig:efficiency}
\end{figure}

\noindent\textbf{Exp-3: Efficiency Evaluation.}
Figure~\ref{fig:efficiency} compares the training and evaluation efficiency of HyperJoin against baselines.
For the training phase, HyperJoin completes within 3.8--31.6 minutes across the four datasets (e.g., 31.6 minutes on CAN), which is substantially faster than DeepJoin that requires hours (e.g., 6.9 hours on CAN; 11.5 hours on USA).
This is expected since DeepJoin performs model fine-tuning (based on DistilBERT), which introduces a significantly higher training cost than lightweight training objectives.
In addition, OmniMatch incurs a substantial offline overhead for constructing column similarity graphs due to extensive pairwise column operations; in our setting, this preprocessing alone takes on the order of hours (e.g., 30.94 hours on USA and 43.37 hours on CAN).
For the ranking phase, HyperJoin answers queries in 8--15 seconds, and is consistently faster than DeepJoin on larger datasets, while remaining in the same order of magnitude as LSH and Snoopy.
Note that the reported DeepJoin time is measured end-to-end and includes rebuilding the target-side index for each run, rather than only the online top-$k$ retrieval step.

% {\color{red}TODO: The efficiency results are missing. Moreover, the following results can be part of the ablation study.}

\begin{figure*}[t]
    \centering
    \includegraphics[width=0.96\textwidth]{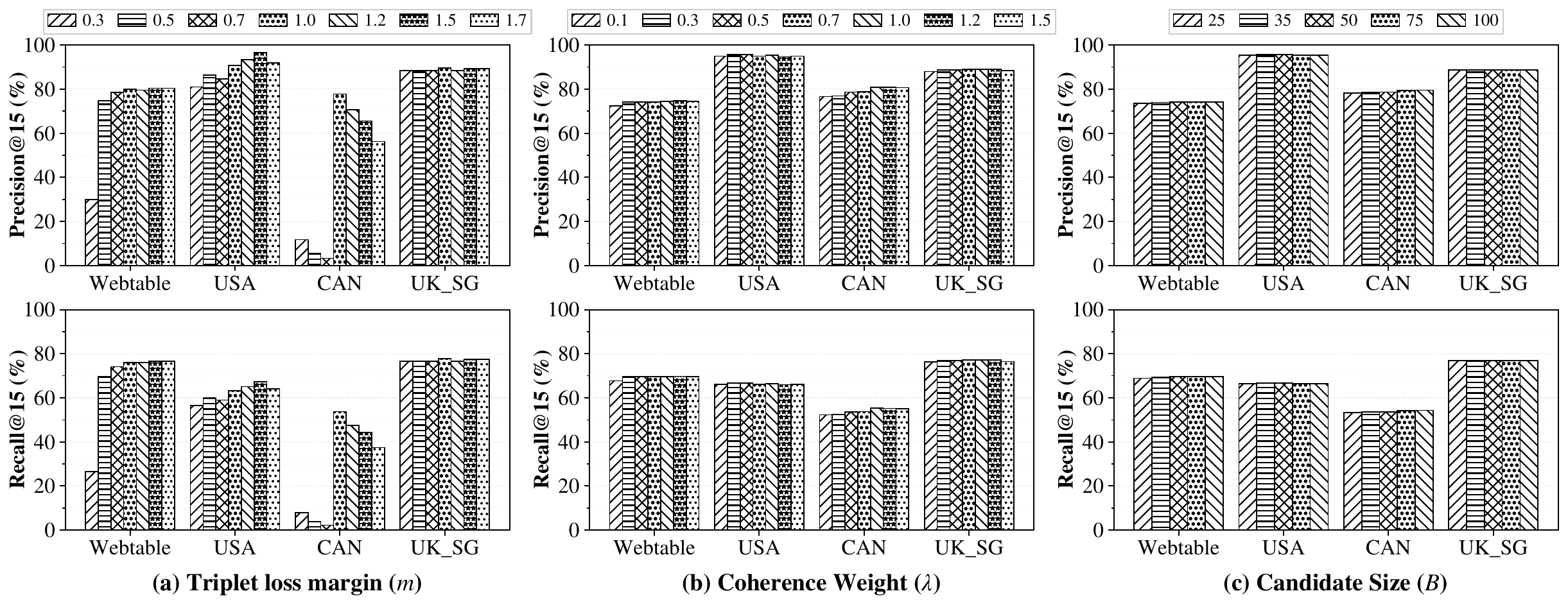}
    \vspace{-4pt}
    \caption{Hyperparameter sensitivity analysis on datasets.
    (a) Triplet loss margin $m$, (b) Global Cohenrent Seaech coherence weight $\lambda$, and (c) Global Cohenrent Seaech candidate size $B$.}
    \vspace{-6pt}
    \label{fig:hyperparam_sensitivity}
\end{figure*}

\subsection{Robustness Evaluation}

\noindent\textbf{Exp-4: Triplet Loss Margin $m$.}
We evaluate triplet loss margins $m \in \{0.3, 0.5, 0.7, 1.0, 1.2, 1.5, 1.7\}$ across four datasets.
% Figure~\ref{fig:hyperparam_sensitivity} shows that the margin is a critical hyperparameter:
As shown in Figure~\ref{fig:hyperparam_sensitivity}(a), the triplet loss margin is critical.
Too small margins can trigger severe degradation or even collapse, particularly on CAN and Webtable (e.g., CAN drops to 3.33\% P@15 / 2.25\% R@15 at $m=0.7$ and Webtable to 30.00\% P@15 / 26.63\% R@15 at $m=0.3$).
Performance becomes stable at $m=1.0$, yielding consistently reliable results across datasets and avoiding catastrophic failures.
Larger margins can further benefit some corpora (e.g., $m=1.5$ performs best on USA at 96.67\% / 67.36\% and also improves Webtable to 80.44\% / 76.52\%), but may slightly hurt CAN and brings only minor changes on UK\_SG (peak at $m=1.0$ with 89.67\% / 77.71\%).
Overall, we use $m=1.0$ as a robust default in all experiments.

\noindent\textbf{Exp-5: Coherence Weight $\lambda$.}
We evaluate $\lambda \in \{0.1, 0.3, 0.5, 0.7, 1.0,\allowbreak 1.2, 1.5\}$ across four datasets.
Figure~\ref{fig:hyperparam_sensitivity}(b) indicates that dataset-dependent optima, but a stable plateau around $\lambda \in [0.7, 1.2]$.
Within this range, UK\_SG and CAN generally perform best, while Webtable improves steadily up to $\lambda=1.2$ (74.67\% P@15 / 69.74\% R@15) and then changes little.
In contrast, USA prefers smaller $\lambda$ (best around $\lambda=0.3$--$0.5$ with 95.67\% / 66.66\%) and degrades mildly for larger values, suggesting over-regularization by coherence.
CAN gains substantially from coherence and peaks at $\lambda=1.0$--$1.2$ (80.89\% P@15 and 55.26\% R@15 at $\lambda=1.0$), followed by a slight decline.
Overall, we set $\lambda=1.0$ as a robust default, delivering near-optimal results on three datasets while remaining competitive on USA.

\noindent\textbf{Exp-6: Candidate Size $B$.}
We vary the candidate pool sizes $B \in \{25, 35, 50, 75, 100\}$ for Global Cohenrent Seaech.
As shown in Figure~\ref{fig:hyperparam_sensitivity}(c), performance is largely stable w.r.t.\ $B$,
with only minor fluctuations across datasets (max variation $\le$1.34 in P@15 and $\le$0.94 in R@15).
In particular, results quickly saturate at moderate pool sizes: UK\_SG is unchanged across all $B$,
and Webtable/USA plateau around $B \approx 35$--$50$.
CAN shows a mild monotonic improvement with larger pools, but the gains beyond $B=50$ are marginal.
Overall, we set $B=50$ as the default, balancing near-optimal quality and computational cost.

% \vspace{-2mm}
\section{Related Work}
% \vspace{-1mm}
\label{sec:Relatedwork}

Existing methods for joinable table discovery \cite{bogatu2020dataset,deng2024lakebench,dong2021efficient,fan2023semantics,cong2023pylon,hu2023automatic} can be broadly classified into traditional set-overlap methods and deep learning-based methods.
Traditional approaches, such as Jaccard similarity~\cite{fernandez2019lazo}, containment coefficients~\cite{deng2017silkmoth,yakout2012infogather,zhang2010multi}, and LSH~\cite{zhu2016lsh}, measure column overlap using set-based metrics and exact or fuzzy string matching.
While scalable, these methods rely on syntactic matching and fail to capture semantic relationships. Deep learning-based methods address this limitation through two main paradigms. Pre-trained language model approaches, such as DeepJoin~\cite{dong2023deepjoin} and Snoopy~\cite{guo2025snoopy}, employ transformers (e.g., BERT) to encode column values into semantic embeddings. Graph neural network approaches, such as OmniMatch~\cite{koutras2025omnimatch}, model table repositories as multi-relational graphs and use RGCN~\cite{schlichtkrull2018modeling} to learn structure-aware column embeddings.
To the best of our knowledge, HyperJoin is the first hypergraph-based framework for joinable table discovery. It models both intra-table and inter-table structural interactions in a unified representation learning phase, and further introduces a MST-based global coherent reranking strategy to return globally consistent join candidates.

% \vspace{-2mm}

\section{Conclusion}
% \vspace{-1mm}
\label{sec:Conclusion}

In this paper, we address the problem of joinable table discovery in data lakes and propose HyperJoin, a hypergraph-based framework.
HyperJoin follows a two-phase design with offline representation learning and an online ranking.
In the offline phase, we propose an LLM-augmented hypergraph that models both intra-table structure and inter-table joinability via dual-type hyperedges, and learns expressive column representations with HIN.
In the online phase, we introduce  efficient global coherent reranking module for a globally coherent result set, balancing scalability and effectiveness.
Extensive experiments on four public benchmarks show that HyperJoin consistently outperforms state-of-the-art methods, achieving superior Precision and Recall overall, while returning more coherent and semantically relevant join candidates.

%\clearpage

\newpage
\bibliographystyle{ACM-Reference-Format}
\bibliography{sample}

@article{radford2018improving,
  title={Improving language understanding by generative pre-training},
  author={Radford, Alec},
  journal={Preprint},
  year={2018}
}

@article{deng2024lakebench,
  title={Lakebench: A benchmark for discovering joinable and unionable tables in data lakes},
  author={Deng, Yuhao and Chai, Chengliang and Cao, Lei and Yuan, Qin and Chen, Siyuan and Yu, Yanrui and Sun, Zhaoze and Wang, Junyi and Li, Jiajun and Cao, Ziqi and others},
  journal={Proceedings of the VLDB Endowment},
  volume={17},
  number={8},
  pages={1925--1938},
  year={2024},
  publisher={VLDB Endowment}
}

@article{koutras2025omnimatch,
  title={OmniMatch: Joinability Discovery in Data Products},
  author={Koutras, Christos and Zhang, Jiani and Qin, Xiao and Lei, Chuan and Ioannidis, Vasileios and Faloutsos, Christos and Karypis, George and Katsifodimos, Asterios},
  journal={Proceedings of the VLDB Endowment},
  volume={18},
  number={11},
  pages={4588--4601},
  year={2025},
  publisher={VLDB Endowment}
}

@inproceedings{he2023generalization,
  title={A generalization of vit/mlp-mixer to graphs},
  author={He, Xiaoxin and Hooi, Bryan and Laurent, Thomas and Perold, Adam and LeCun, Yann and Bresson, Xavier},
  booktitle={International conference on machine learning},
  pages={12724--12745},
  year={2023},
  organization={PMLR}
}

@article{weisfeiler1968reduction,
  title={The reduction of a graph to canonical form and the algebra which appears therein},
  author={Weisfeiler, Boris and Leman, Andrei},
  journal={nti, Series},
  volume={2},
  number={9},
  pages={12--16},
  year={1968}
}

@article{tolstikhin2021mlp,
  title={Mlp-mixer: An all-mlp architecture for vision},
  author={Tolstikhin, Ilya O and Houlsby, Neil and Kolesnikov, Alexander and Beyer, Lucas and Zhai, Xiaohua and Unterthiner, Thomas and Yung, Jessica and Steiner, Andreas and Keysers, Daniel and Uszkoreit, Jakob and others},
  journal={Advances in neural information processing systems},
  volume={34},
  pages={24261--24272},
  year={2021}
}

@article{dong2023deepjoin,
  title={DeepJoin: Joinable Table Discovery with Pre-Trained Language Models},
  author={Dong, Yuyang and Xiao, Chuan and Nozawa, Takuma and Enomoto, Masafumi and Oyamada, Masafumi},
  journal={Proceedings of the VLDB Endowment},
  volume={16},
  number={10},
  pages={2458--2470},
  year={2023},
  publisher={VLDB Endowment}
}

@article{guo2025snoopy,
  title={Snoopy: Effective and Efficient Semantic Join Discovery via Proxy Columns},
  author={Guo, Yuxiang and Mao, Yuren and Hu, Zhonghao and Chen, Lu and Gao, Yunjun},
  journal={IEEE Transactions on Knowledge and Data Engineering},
  year={2025},
  publisher={IEEE}
}

@article{lai2025auto,
  title={Auto-Prep: Holistic Prediction of Data Preparation Steps for Self-Service Business Intelligence},
  author={Lai, Eugenie Y and He, Yeye and Chaudhuri, Surajit},
  journal={arXiv preprint arXiv:2504.11627},
  year={2025}
}

@inproceedings{zhu2019josie,
  title={Josie: Overlap set similarity search for finding joinable tables in data lakes},
  author={Zhu, Erkang and Deng, Dong and Nargesian, Fatemeh and Miller, Ren{\'e}e J},
  booktitle={Proceedings of the 2019 International Conference on Management of Data},
  pages={847--864},
  year={2019}
}

@article{zhu2016lsh,
  title={LSH Ensemble: Internet-Scale Domain Search},
  author={Zhu, Erkang and Nargesian, Fatemeh and Pu, Ken Q and Miller, Ren{\'e}e J},
  journal={Proceedings of the VLDB Endowment},
  volume={9},
  number={12},
  year={2016}
}

@article{fan2023semantics,
  title={Semantics-Aware Dataset Discovery from Data Lakes with Contextualized Column-Based Representation Learning},
  author={Fan, Grace and Wang, Jin and Li, Yuliang and Zhang, Dan and Miller, Ren{\'e}e J},
  journal={Proceedings of the VLDB Endowment},
  volume={16},
  number={7},
  pages={1726--1739},
  year={2023},
  publisher={VLDB Endowment}
}

@inproceedings{dong2021efficient,
  title={Efficient joinable table discovery in data lakes: A high-dimensional similarity-based approach},
  author={Dong, Yuyang and Takeoka, Kunihiro and Xiao, Chuan and Oyamada, Masafumi},
  booktitle={2021 IEEE 37th International Conference on Data Engineering (ICDE)},
  pages={456--467},
  year={2021},
  organization={IEEE}
}

@inproceedings{devlin2019bert,
  title={Bert: Pre-training of deep bidirectional transformers for language understanding},
  author={Devlin, Jacob and Chang, Ming-Wei and Lee, Kenton and Toutanova, Kristina},
  booktitle={Proceedings of the 2019 conference of the North American chapter of the association for computational linguistics: human language technologies, volume 1 (long and short papers)},
  pages={4171--4186},
  year={2019}
}

@article{sanh2019distilbert,
  title={DistilBERT, a distilled version of BERT: smaller, faster, cheaper and lighter},
  author={Sanh, Victor and Debut, Lysandre and Chaumond, Julien and Wolf, Thomas},
  journal={arXiv preprint arXiv:1910.01108},
  year={2019}
}

@article{song2020mpnet,
  title={Mpnet: Masked and permuted pre-training for language understanding},
  author={Song, Kaitao and Tan, Xu and Qin, Tao and Lu, Jianfeng and Liu, Tie-Yan},
  journal={Advances in neural information processing systems},
  volume={33},
  pages={16857--16867},
  year={2020}
}

@misc{opendata,
  title        = {Government Open Data Portals (USA, CAN, UK, SG)},
  howpublished = {\url{https://www.data.gov/}; \url{https://open.canada.ca/}; \url{https://data.gov.uk/}; \url{https://data.gov.sg/}}
}

@misc{webtable,
  title = {WebTable},
  howpublished = {\url{https://webdatacommons.org/webtables/}}
}

@inproceedings{bhagavatula2015tabel,
  title={Tabel: Entity linking in web tables},
  author={Bhagavatula, Chandra Sekhar and Noraset, Thanapon and Downey, Doug},
  booktitle={International Semantic Web Conference},
  pages={425--441},
  year={2015},
  organization={Springer}
}

@inproceedings{wang2021retrieving,
  title={Retrieving complex tables with multi-granular graph representation learning},
  author={Wang, Fei and Sun, Kexuan and Chen, Muhao and Pujara, Jay and Szekely, Pedro},
  booktitle={Proceedings of the 44th International ACM SIGIR Conference on Research and Development in Information Retrieval},
  pages={1472--1482},
  year={2021}
}

@inproceedings{DBLP:conf/acl/YinNYR20,
  author       = {Pengcheng Yin and
                  Graham Neubig and
                  Wen{-}tau Yih and
                  Sebastian Riedel},
  title        = {TaBERT: Pretraining for Joint Understanding of Textual and Tabular
                  Data},
  booktitle    = {{ACL}},
  pages        = {8413--8426},
  publisher    = {Association for Computational Linguistics},
  year         = {2020}
}

@inproceedings{fernandez2019lazo,
  title={Lazo: A cardinality-based method for coupled estimation of jaccard similarity and containment},
  author={Fernandez, Raul Castro and Min, Jisoo and Nava, Demitri and Madden, Samuel},
  booktitle={2019 IEEE 35th International Conference on Data Engineering (ICDE)},
  pages={1190--1201},
  year={2019},
  organization={IEEE}
}

@article{deng2017silkmoth,
  title={SILKMOTH: An Efficient Method for Finding Related Sets with Maximum Matching Constraints},
  author={Deng, Dong and Kim, Albert and Madden, Samuel and Stonebraker, Michael},
  journal={Proceedings of the VLDB Endowment},
  volume={10},
  number={10},
  year={2017}
}

@inproceedings{yakout2012infogather,
  title={Infogather: entity augmentation and attribute discovery by holistic matching with web tables},
  author={Yakout, Mohamed and Ganjam, Kris and Chakrabarti, Kaushik and Chaudhuri, Surajit},
  booktitle={Proceedings of the 2012 ACM SIGMOD International Conference on Management of Data},
  pages={97--108},
  year={2012}
}

@article{zhang2010multi,
  title={On multi-column foreign key discovery},
  author={Zhang, Meihui and Hadjieleftheriou, Marios and Ooi, Beng Chin and Procopiuc, Cecilia M and Srivastava, Divesh},
  journal={Proceedings of the VLDB Endowment},
  volume={3},
  number={1-2},
  pages={805--814},
  year={2010},
  publisher={VLDB Endowment}
}

@inproceedings{schlichtkrull2018modeling,
  title={Modeling relational data with graph convolutional networks},
  author={Schlichtkrull, Michael and Kipf, Thomas N and Bloem, Peter and Van Den Berg, Rianne and Titov, Ivan and Welling, Max},
  booktitle={European semantic web conference},
  pages={593--607},
  year={2018},
  organization={Springer}
}

@misc{hyperjoin_full,
  title        = {{HyperJoin}: Full Version},
  author       = {Shiyuan Liu and Jianwei Wang and Xuemin Lin and Lu Qin and Wenjie Zhang and Ying Zhang},
  year         = {2026},
  howpublished = {\url{https://github.com/T-Lab/HyperJoin/tree/main/full_version}}
}

@inproceedings{chai2020human,
  title={Human-in-the-loop outlier detection},
  author={Chai, Chengliang and Cao, Lei and Li, Guoliang and Li, Jian and Luo, Yuyu and Madden, Samuel},
  booktitle={Proceedings of the 2020 ACM SIGMOD international conference on management of data},
  pages={19--33},
  year={2020}
}

@inproceedings{deng2017data,
  title={The Data Civilizer System.},
  author={Deng, Dong and Fernandez, Raul Castro and Abedjan, Ziawasch and Wang, Sibo and Stonebraker, Michael and Elmagarmid, Ahmed K and Ilyas, Ihab F and Madden, Samuel and Ouzzani, Mourad and Tang, Nan},
  booktitle={Cidr},
  year={2017}
}

@inproceedings{deng2024misdetect,
  title={Misdetect: Iterative mislabel detection using early loss},
  author={Deng, Yuhao and Chai, Chengliang and Cao, Lei and Tang, Nan and Wang, Jiayi and Fan, Ju and Yuan, Ye and Wang, Guoren},
  year={2024},
  organization={Association for Computing Machinery (ACM)}
}

@article{nargesian2019data,
  title={Data lake management: challenges and opportunities},
  author={Nargesian, Fatemeh and Zhu, Erkang and Miller, Ren{\'e}e J and Pu, Ken Q and Arocena, Patricia C},
  journal={Proceedings of the VLDB Endowment},
  volume={12},
  number={12},
  pages={1986--1989},
  year={2019},
  publisher={VLDB Endowment}
}

@inproceedings{khattab2020colbert,
  title={Colbert: Efficient and effective passage search via contextualized late interaction over bert},
  author={Khattab, Omar and Zaharia, Matei},
  booktitle={Proceedings of the 43rd International ACM SIGIR conference on research and development in Information Retrieval},
  pages={39--48},
  year={2020}
}

@article{zhang2023large,
  title={Large language models as data preprocessors},
  author={Zhang, Haochen and Dong, Yuyang and Xiao, Chuan and Oyamada, Masafumi},
  journal={arXiv:2308.16361},
  year={2023}
}

@article{jegou2010product,
  title={Product quantization for nearest neighbor search},
  author={Jegou, Herve and Douze, Matthijs and Schmid, Cordelia},
  journal={IEEE transactions on pattern analysis and machine intelligence},
  volume={33},
  number={1},
  pages={117--128},
  year={2010},
  publisher={IEEE}
}

@article{malkov2018efficient,
  title={Efficient and robust approximate nearest neighbor search using hierarchical navigable small world graphs},
  author={Malkov, Yu A and Yashunin, Dmitry A},
  journal={IEEE transactions on pattern analysis and machine intelligence},
  volume={42},
  number={4},
  pages={824--836},
  year={2018},
  publisher={IEEE}
}

@inproceedings{bogatu2020dataset,
  title={Dataset discovery in data lakes},
  author={Bogatu, Alex and Fernandes, Alvaro AA and Paton, Norman W and Konstantinou, Nikolaos},
  booktitle={2020 ieee 36th international conference on data engineering (icde)},
  pages={709--720},
  year={2020},
  organization={IEEE}
}

@article{wang2024efficient,
  title={Efficient unsupervised community search with pre-trained graph transformer},
  author={Wang, Jianwei and Wang, Kai and Lin, Xuemin and Zhang, Wenjie and Zhang, Ying},
  journal={arXiv preprint arXiv:2403.18869},
  year={2024}
}

@article{wang2024neural,
  title={Neural attributed community search at billion scale},
  author={Wang, Jianwei and Wang, Kai and Lin, Xuemin and Zhang, Wenjie and Zhang, Ying},
  journal={Proceedings of the ACM on Management of Data},
  volume={1},
  number={4},
  pages={1--25},
  year={2024},
  publisher={ACM New York, NY, USA}
}

@article{wang2024missing,
  title={Missing data imputation with uncertainty-driven network},
  author={Wang, Jianwei and Zhang, Ying and Wang, Kai and Lin, Xuemin and Zhang, Wenjie},
  journal={Proceedings of the ACM on Management of Data},
  volume={2},
  number={3},
  pages={1--25},
  year={2024},
  publisher={ACM New York, NY, USA}
}

@article{wang2025llm,
  title={On llm-enhanced mixed-type data imputation with high-order message passing},
  author={Wang, Jianwei and Wang, Kai and Zhang, Ying and Zhang, Wenjie and Xu, Xiwei and Lin, Xuemin},
  journal={arXiv preprint arXiv:2501.02191},
  year={2025}
}

@article{cong2023pylon,
  title={Pylon: Semantic table union search in data lakes},
  author={Cong, Tianji and Nargesian, Fatemeh and Jagadish, HV},
  journal={arXiv preprint arXiv:2301.04901},
  year={2023}
}

@inproceedings{hu2023automatic,
  title={Automatic table union search with tabular representation learning},
  author={Hu, Xuming and Wang, Shen and Qin, Xiao and Lei, Chuan and Shen, Zhengyuan and Faloutsos, Christos and Katsifodimos, Asterios and Karypis, George and Wen, Lijie and Yu, Philip S},
  booktitle={Findings of the Association for Computational Linguistics: ACL 2023},
  pages={3786--3800},
  year={2023}
}

@article{khatiwada2022integrating,
  title={Integrating data lake tables},
  author={Khatiwada, Aamod and Shraga, Roee and Gatterbauer, Wolfgang and Miller, Ren{\'e}e J},
  journal={Proceedings of the VLDB Endowment},
  volume={16},
  number={4},
  pages={932--945},
  year={2022},
  publisher={VLDB Endowment}
}

@article{lee2020hypergraph,
  title={Hypergraph motifs: concepts, algorithms, and discoveries},
  author={Lee, Geon and Ko, Jihoon and Shin, Kijung},
  journal={VLDB},
  volume={13},
  number={12},
  pages={2256--2269},
  year={2020},
  publisher={VLDB Endowment}
}

@article{arafat2023neighborhood,
  title={Neighborhood-based Hypergraph Core Decomposition},
  author={Arafat, Naheed Anjum and Khan, Arijit and Rai, Arpit Kumar and Ghosh, Bishwamittra},
  journal={Proceedings of the VLDB Endowment},
  volume={16},
  number={9},
  pages={2061--2074},
  year={2023},
  publisher={VLDB Endowment}
}

@article{lugo2021classification,
  title={Classification in biological networks with hypergraphlet kernels},
  author={Lugo-Martinez, Jose and Zeiberg, Daniel and Gaudelet, Thomas and Malod-Dognin, No{\"e}l and Przulj, Natasa and Radivojac, Predrag},
  journal={Bioinformatics},
  volume={37},
  number={7},
  pages={1000--1007},
  year={2021},
  publisher={Oxford University Press}
}

@article{wu2022hypergraph,
  title={Hypergraph collaborative network on vertices and hyperedges},
  author={Wu, Hanrui and Yan, Yuguang and Ng, Michael Kwok-Po},
  journal={IEEE Transactions on Pattern Analysis and Machine Intelligence},
  volume={45},
  number={3},
  pages={3245--3258},
  year={2022},
  publisher={IEEE}
}

@inproceedings{deac2022expander,
  title={Expander graph propagation},
  author={Deac, Andreea and Lackenby, Marc and Veli{\v{c}}kovi{\'c}, Petar},
  booktitle={Learning on Graphs Conference},
  pages={38--1},
  year={2022},
  organization={PMLR}
}

@article{arnaiz2022diffwire,
  title={Diffwire: Inductive graph rewiring via the lov$\backslash$'asz bound},
  author={Arnaiz-Rodr{\'\i}guez, Adri{\'a}n and Begga, Ahmed and Escolano, Francisco and Oliver, Nuria},
  journal={arXiv preprint arXiv:2206.07369},
  year={2022}
}

@article{wang2024simpler,
  title={Simpler is More: Efficient Top-K Nearest Neighbors Search on Large Road Networks},
  author={Wang, Yiqi and Yuan, Long and Zhang, Wenjie and Chen, Zi and Lin, Xuemin and Liu, Qing},
  journal={Proceedings of the VLDB Endowment},
  volume={17},
  number={13},
  pages={4683--4695},
  year={2024},
  publisher={VLDB Endowment}
}

\newpage
\section{Appendix}

% \label{sec:forward_pass}
\subsection{Overall forward process of \modelname.}
Algorithm~\ref{algo:sf_hgn} summarizes the complete forward pass of \modelname.
Lines~2-7 inject positional encodings by precomputing Laplacian eigenvectors offline and adding table-level and column-level encodings.
Lines~9-14 apply node-level GNN transformation and aggregate to hyperedge level with domain-specific transforms.
Lines~16-24 perform global hyperedge mixing via structure-aware attention and channel mixing.
Lines~26-29 propagate hyperedge embeddings back to columns with residual connections.
Finally, line~31 returns the final embeddings.

\begin{algorithm}[t]
\caption{\modelfullname~(\modelname)}
\label{algo:sf_hgn}
\LinesNumbered
\DontPrintSemicolon
\KwIn{Hypergraph $\mathcal{H} = (\mathcal{V}, \mathcal{E}, \mathbf{X}^v, \mathbf{X}^e, \boldsymbol{\Pi})$, Table IDs, Joinable pairs $\mathcal{P}_{\text{join}}$}
\KwOut{Final column embeddings $\{\mathbf{h}_v^{\text{final}}\}_{v \in \mathcal{V}}$}

\tcp{\textbf{Stage 1: Positional Encoding}}
Precompute Laplacian eigenvectors (offline):\;
\Indp
$\mathbf{A} \leftarrow \text{BuildAdjacencyMatrix}(\mathcal{P}_{\text{join}})$\;
$\mathbf{L}_{\text{norm}} = \mathbf{I}_N - \mathbf{D}^{-1/2}\mathbf{A}\mathbf{D}^{-1/2}
$\;
$\mathbf{V}_{\text{pe}} \leftarrow \text{EigenVectors}(\mathbf{L}_{\text{norm}}, k=16)$\;
\Indm
Inject table-level and column-level positional encodings:\;
\Indp
\For{each column/node $v$ with table ID $t$ and index $i$}{
    $\mathbf{h}_v^{(0)} \leftarrow \mathbf{x}_v^v + \alpha \cdot \mathbf{E}_{\text{tbl}}[t] + \beta \cdot \text{MLP}_{\text{pe}}(\mathbf{V}_{\text{pe}}[i,:])$\;
}
\Indm

\tcp{\textbf{Stage 2: Local Hyperedge Aggregation}}
$\mathbf{H} \leftarrow [\mathbf{h}_{v_1}^{(0)}, \ldots, \mathbf{h}_{v_N}^{(0)}]^T$\;
\For{$\ell = 1$ to $L=2$}{
    $\mathbf{H} \leftarrow \text{LayerNorm}(\text{ReLU}(\mathbf{H}\mathbf{W}^{(\ell)} + \mathbf{b}^{(\ell)}))$\;
}
Aggregate to hyperedge level with domain-specific transforms:\;
\Indp
$\mathbf{d}_e \leftarrow \boldsymbol{\Pi}^T \mathbf{1}_N$ \tcp*{Hyperedge degrees}
$\mathbf{D}_e \leftarrow \text{Diag}(\mathbf{d}_e)$\;
$\mathbf{X}^e \leftarrow \mathbf{D}_e^{-1}\boldsymbol{\Pi}^T \mathbf{H}$\;
$\mathbf{Z}^{(0)} \leftarrow \mathbf{X}^e$ \tcp*{mixer input tokens}
Partition $\mathbf{Z}^{(0)}$ into $\mathbf{Z}^{(0)}_{\text{inter}}$ and $\mathbf{Z}^{(0)}_{\text{intra}}$\;
$\tilde{\mathbf{Z}}_{\text{inter}} \leftarrow \mathbf{Z}^{(0)}_{\text{inter}}\mathbf{W}_{\text{inter}}$\;
$\tilde{\mathbf{Z}}_{\text{intra}} \leftarrow \mathbf{Z}^{(0)}_{\text{intra}}\mathbf{W}_{\text{intra}}$\;
$\tilde{\mathbf{Z}} \leftarrow [\tilde{\mathbf{Z}}_{\text{inter}};\tilde{\mathbf{Z}}_{\text{intra}}]$\;
\Indm

\tcp{\textbf{Stage 3: Global Hyperedge Mixing}}
$\mathbf{A}_{\text{hyperedge}} \leftarrow \text{RowNormalize}(\boldsymbol{\Pi}^T \boldsymbol{\Pi} - \text{diag}(\boldsymbol{\Pi}^T \boldsymbol{\Pi}))$\;

\For{$\ell = 1$ to $L_{\text{mixer}}=2$}{
    % \tcp{Token mixing with structure bias}
    \tcp{\HIname~with structure bias}
    $\mathbf{Q}, \mathbf{K}, \mathbf{V} \leftarrow \text{LayerNorm}(\tilde{\mathbf{Z}}) [\mathbf{W}_Q, \mathbf{W}_K, \mathbf{W}_V]$\;
    $\mathbf{Z}_{\text{HI}} \leftarrow \tilde{\mathbf{Z}} + \text{MultiHead}\!\left(\text{softmax}\!\left(\frac{\mathbf{Q}\mathbf{K}^T}{\sqrt{d_k}} + \lambda \mathbf{A}_{\text{hyperedge}}\right)\mathbf{V}\right)\mathbf{W}_O$\;
    \tcp{\HRname}
    $\tilde{\mathbf{Z}} \leftarrow \mathbf{Z}_{\text{HI}} + \mathbf{W}_2 \text{GELU}(\mathbf{W}_1 \text{LayerNorm}(\mathbf{Z}_{\text{HI}})^T)^T$\;
}
$\mathbf{Z}_{\text{out}} \leftarrow \tilde{\mathbf{Z}}$ \tcp*{final hyperedge embeddings}
Let $\mathbf{z}_{e_j}$ denote the $j$-th row of $\mathbf{Z}_{\text{out}}$.\;

\tcp{\textbf{Propagate hyperedge back to columns}}
\For{$i=1$ to $N$}{
    $\mathcal{N}(v_i) \leftarrow \{ j \mid \boldsymbol{\Pi}_{ij}=1 \}$\;
    $\mathbf{h}_{v_i}^{\text{hyperedge}} \leftarrow \frac{1}{|\mathcal{N}(v_i)|} \sum_{j \in \mathcal{N}(v_i)} \mathbf{W}_{\text{h2c}} \mathbf{z}_{e_j}$\;
    $\mathbf{h}_{v_i}^{\text{final}} \leftarrow \text{L2Norm}(\text{LayerNorm}(\mathbf{h}_{v_i} + \mathbf{h}_{v_i}^{\text{hyperedge}}))$\;
}

\Return $\{\mathbf{h}_v^{\text{final}}\}_{v \in \mathcal{V}}$\;
\end{algorithm}

% \subsection{Theorems and Proof}
\subsection{Proofs of Theoretical Results}

This section collects proofs deferred from the main paper. For readability, we group proofs by the section where the corresponding theorem appears.

\subsubsection{Proof of Theorem~\ref{thm:bayes_risk}}

% \paragraph{Setup shared by Section 4.1.3 results.}
\paragraph{Setup.}
Let $Y_{ij}\in\{0,1\}$ denote whether two columns $(v_i,v_j)$ are joinable.
Let $\mathbf{x}_i$ be the initial feature of column $v_i$.
For each column $v_i$, define two context variables induced by the constructed hypergraph:
(i) its intra-table context $\mathcal{N}^{\text{intra}}(v_i)$ (the set of columns co-located with $v_i$ in the same table),
and (ii) its inter-table context $\mathcal{N}^{\text{inter}}(v_i)$ (the set of columns connected with $v_i$ through joinable
connectivity, e.g., within the same join-connected component).
We write the available information for deciding joinability as a feature set:
\begin{equation}
\begin{aligned}
\mathcal{F}_{\text{single}}(i,j) &= \{\mathbf{x}_i,\mathbf{x}_j\},\\
\mathcal{F}_{\text{multi}}(i,j) &=
\Big\{\mathbf{x}_i,\mathbf{x}_j,\,
\mathcal{N}^{\text{intra}}(v_i),\mathcal{N}^{\text{intra}}(v_j),\\
&\hspace{18pt}
\mathcal{N}^{\text{inter}}(v_i),\mathcal{N}^{\text{inter}}(v_j)\Big\}.
\end{aligned}
\end{equation}

Let $\Psi^*(\mathcal{F})$ denote the minimum achievable expected discovery risk under feature set $\mathcal{F}$:
\[
\Psi^*(\mathcal{F}) \;=\; \inf_{g}\; \mathbb{E}\big[\ell(g(\mathcal{F}),Y_{ij})\big],
\]
where $g$ is any measurable scoring function and $\ell(\cdot)$ is a bounded loss (e.g., logistic or 0-1 loss).

\bayesrisk*

\begin{proof}
Any predictor $g_{\text{single}}(\mathcal{F}_{\text{single}})$ is also feasible under $\mathcal{F}_{\text{multi}}$
by defining $g_{\text{multi}}(\mathcal{F}_{\text{multi}})=g_{\text{single}}(\mathcal{F}_{\text{single}})$,
i.e., simply ignoring the additional context variables. Therefore,
\begin{equation}
\begin{aligned}
\Psi^*(\mathcal{F}_{\text{multi}})
&=\inf_{g_{\text{multi}}}\mathbb{E}\!\left[\ell\!\left(g_{\text{multi}}(\mathcal{F}_{\text{multi}}),Y_{ij}\right)\right] \\
&\le \inf_{g_{\text{single}}}\mathbb{E}\!\left[\ell\!\left(g_{\text{single}}(\mathcal{F}_{\text{single}}),Y_{ij}\right)\right]
=\Psi^*(\mathcal{F}_{\text{single}}).
\end{aligned}
\end{equation}

If the conditional distribution of $Y_{ij}$ changes when conditioning on the added contexts,
then the Bayes decision rule under $\mathcal{F}_{\text{multi}}$ differs from that under $\mathcal{F}_{\text{single}}$ on a set of non-zero measure,
yielding strictly smaller Bayes risk.
\end{proof}

\subsubsection{Proof of Theorem~\ref{thm:prior_special}.}
\priorspecial*

\begin{proof}
% {\color{blue}
This can be achieved by choosing $\mathcal{E}$ as $N$ singleton hyperedges $\mathcal{E}=\{\{v_1\},\ldots,\{v_N\}\}$ (or equivalently disabling hyperedge message passing).
If every hyperedge is singleton, then each column belongs to exactly one hyperedge of size 1.
Any within-hyperedge aggregation (mean/sum/attention) over a singleton returns the column itself,
thus Local Hyperedge Aggregation leaves representations unchanged.
Moreover, since hyperedges share no nodes, hyperedge adjacency becomes diagonal and Global Hyperedge Mixing
degenerates to per-hyperedge (thus per-column) transformations without cross-column exchange.
Therefore, the resulting embedding has the form $\mathbf{h}_i = \phi_\theta(\mathbf{x}_i)$.
DeepJoin instantiates $\phi_\theta$ as a PLM-based column encoder trained with self-supervised contrastive learning,
and Snoopy further applies a deterministic proxy-based transformation on top of such column-wise embeddings.
Hence both methods are recovered as restricted instances that do not leverage $\mathcal{N}^{\text{intra}}/\mathcal{N}^{\text{inter}}$. \end{proof}

% \subsubsection{Proofs for Section 4.2.4}
\subsubsection{Proof of Theorem~\ref{thm:expressiveness}}

\expressiveness*

\begin{proof}
We first note that $\mathcal{H}_{\text{base}} \subseteq \mathcal{H}_{\text{pe}}$ since any encoder in $\mathcal{H}_{\text{base}}$
can be realized in $\mathcal{H}_{\text{pe}}$ by ignoring the positional terms (e.g., setting $\alpha=\beta=0$).

Since the encoder in $\mathcal{H}_{\text{base}}$ is permutation-equivariant and $u,v$ are 1-WL-indistinguishable with identical raw features,
it must output identical representations for $u$ and $v$ at every layer; hence no hypothesis in $\mathcal{H}_{\text{base}}$ can separate $u$ and $v$.

In contrast, under $\mathcal{H}_{\text{pe}}$, the augmented inputs satisfy
\[
  \mathbf{h}_u=\mathbf{h}^{(0)}_u+\text{PE}(u)\neq \mathbf{h}^{(0)}_v+\text{PE}(v)=\mathbf{h}_v.
\]
With an injective node-wise update, this distinction is preserved, so there exists a hypothesis in $\mathcal{H}_{\text{pe}}$ whose outputs separate $u$ and $v$ (e.g., via a linear readout).
Therefore, $\mathcal{H}_{\text{base}} \subsetneq \mathcal{H}_{\text{pe}}$.
\end{proof}

\subsubsection{Proof of Theorem~\ref{thm:np_hard_formulation}}

\nphard*

\begin{proof}
We prove NP-hardness by reduction from the $0$--$1$ Knapsack problem.
Given a knapsack instance with items $\{1,\ldots,n\}$, each item $i$ having integer weight $w_i$ and value $v_i$, and a capacity $W$, the goal is to choose a subset of items with total weight at most $W$ that maximizes the total value.

We construct an instance of our selection problem as follows.
We create a root node corresponding to the query column $C_q$.
For each item $i$, we create a path of $w_i$ candidate nodes $\{u_{i,1},\ldots,u_{i,w_i}\}$.
We add an edge between $C_q$ and $u_{i,1}$ with weight $0$.
For each $j=1,\ldots,w_i-2$, we add an edge $(u_{i,j},u_{i,j+1})$ with weight $0$.
We assign the item value to the last edge on the path by setting the edge weight $(u_{i,w_i-1},u_{i,w_i})$ to be $v_i$.
We further add $W$ dummy candidate nodes $\{d_1,\ldots,d_W\}$ and connect each $d_t$ to $C_q$ with an edge weight $0$.
We set the result size to $K=W$ and set all query relevance weights to zero, so the objective is determined solely by the coherence term.

We claim that in the induced subgraph on $\{C_q\}\cup \mathcal{R}$, the coherence value equals the total value of the knapsack items encoded by $\mathcal{R}$.
Because coherence is computed on the induced subgraph, any selected node must be connected to $C_q$ using only nodes in $\mathcal{R}\cup\{C_q\}$.
If $\mathcal{R}$ contains $u_{i,w_i}$, then to make $u_{i,w_i}$ connected to $C_q$ in the induced subgraph, $\mathcal{R}$ must contain the entire prefix $\{u_{i,1},\ldots,u_{i,w_i-1}\}$.
Therefore, obtaining the positive weight $v_i$ on the last edge necessarily consumes exactly $w_i$ selected nodes.
Selecting any strict prefix of the path yields no positive contribution because all edges on the prefix have weight $0$.

Given any feasible knapsack solution $S$ with $\sum_{i\in S} w_i \le W$, we form a selection set $\mathcal{R}$ by including all nodes on the paths of items in $S$ and filling the remaining budget with dummy nodes so that $|\mathcal{R}|=W$.
In the induced subgraph, a maximum spanning tree rooted at $C_q$ includes the last edge of each chosen item path, contributing exactly $\sum_{i\in S} v_i$, and all other required edges have weight $0$, so the coherence equals the knapsack value.
Conversely, from any selection set $\mathcal{R}$ of size $W$, we extract the corresponding item subset $S$ consisting of those indices $i$ for which $u_{i,w_i}\in \mathcal{R}$.
By the connectivity requirement on the induced subgraph, this implies that all $w_i$ nodes on the path are selected, so $\sum_{i\in S} w_i \le W$.
The coherence of $\mathcal{R}$ is then exactly $\sum_{i\in S} v_i$ because positive weights only appear on the last edges of selected item paths.

Thus, solving the selection problem optimally would solve the $0$--$1$ Knapsack problem optimally, which is NP-hard.
Therefore, the coherence-aware top-$K$ selection problem is NP-hard.
\end{proof}

\subsubsection{Proof of Theorem~\ref{thm:surrogate_guarantee}}

\mstlowerbound*

\begin{proof}
Let $T$ be a MaxST of $\mathcal{G}_C[S]$ and $u^*\in\arg\max_{u\in S} w(v,u)$. Then $T\cup\{(v,u^*)\}$ is a spanning tree of $\mathcal{G}_C[S\cup\{v\}]$, so
$\text{MST}(S\cup\{v\})\ge \text{MST}(S)+\max_{u\in S}w(v,u)$.
\end{proof}

\surrogateguarantee*
\begin{proof}
\[
G(\mathcal{R}\cup\{C^*\})-G(\mathcal{R})
= w(C_q,C^*)+\lambda\big(\text{MST}(S\cup\{C^*\})-\text{MST}(S)\big)
\]
and Lemma~\ref{lem:mst_lower_bound} gives
$\text{MST}(S\cup\{C^*\})-\text{MST}(S)\ge \max_{u\in S} w(C^*,u)$, hence the first claim.
The second claim follows since each iteration adds one new vertex and one attachment edge to the connected component rooted at $C_q$, and the MST maximizes weight over all spanning trees.
\end{proof}

\end{document}